%% file: giant.tex
\documentclass{article}

\usepackage[preprint,nonatbib]{nips_2018_modified}

\usepackage[utf8]{inputenc} % allow utf-8 input
\usepackage[T1]{fontenc}    % use 8-bit T1 fonts
\usepackage[colorlinks]{hyperref}       % hyperlinks
\usepackage{url}            % simple URL typesetting
\usepackage{booktabs}       % professional-quality tables
\usepackage{nicefrac}       % compact symbols for 1/2, etc.
\usepackage{microtype}      % microtypography
\usepackage{amsmath,amssymb,amsmath,amsthm,amsfonts}
\usepackage{graphicx,subfigure}
\usepackage{wrapfig}
\input{text/newcommands}

\title{GIANT: Globally Improved Approximate Newton Method for Distributed Optimization}

\author{
  Shusen Wang \\
  University of California at Berkeley \\
  \texttt{wssatzju@gmail.com} \\
  \And
  Farbod Roosta-Khorasani \\
  University of Queensland \\
  \texttt{fred.roosta@uq.edu.au} \\ 
  \And
  Peng Xu \\
  Stanford University \\
  \texttt{pengxu@stanford.edu} \\
  \And
  Michael W.\ Mahoney \\
  University of California at Berkeley \\
  \texttt{mmahoney@stat.berkeley.edu} \\
}

\begin{document}

\maketitle

\begin{abstract}
For distributed computing environment, we consider the empirical risk minimization problem and propose a distributed and communication-efficient Newton-type optimization method. At every iteration, each worker locally finds an Approximate NewTon (ANT) direction, which is sent to the main driver. The main driver, then, averages all the ANT directions received from workers to form a {\it Globally Improved ANT} (GIANT) direction. GIANT is highly communication efficient and naturally exploits the trade-offs between local computations and global communications in that more local computations result in fewer overall rounds of communications. Theoretically, we show that GIANT enjoys an improved convergence rate as compared with first-order methods and existing distributed Newton-type methods. Further, and in sharp contrast with many existing distributed Newton-type methods, as well as popular first-order methods, a highly advantageous practical feature of GIANT is that it only involves one tuning parameter. We conduct large-scale experiments on a computer cluster and, empirically, demonstrate the superior performance of GIANT.
\end{abstract}

%%%%%%%%%%%%%%%%%%%%%%%%%%%%%%%%%%%%%%%%%%%%%%%%%%%%%%%%%%%%%%%%%%%%%%%%%%%%%%
%%%%%%%%%%%%%%%%%%%%%%%%%%%%%%%%%%%%%%%%%%%%%%%%%%%%%%%%%%%%%%%%%%%%%%%%%%%%%%
%%%%%%%%%%%%%%%%%%%%%%%%%%%%%%%%%%%%%%%%%%%%%%%%%%%%%%%%%%%%%%%%%%%%%%%%%%%%%%
%%%%%%%%%%%%%%%%%%%%%%%%%%%%%%%%%%%%%%%%%%%%%%%%%%%%%%%%%%%%%%%%%%%%%%%%%%%%%%

\section{Introduction}
\label{sec:introduction}

The large-scale nature of many modern ``big-data'' problems, arising routinely in science, engineering, financial markets, Internet and social media, etc., poses significant computational as well as storage challenges for machine learning procedures. For example, the scale of data gathered in many applications nowadays typically exceeds the memory capacity of a single machine, which, in turn, makes learning from data ever more challenging.
In this light, several modern parallel (or distributed) computing architectures, 
e.g., MapReduce \cite{dean2008mapreduce}, Apache Spark \cite{zaharia2010spark,meng2016mllib}, 
GraphLab \cite{low2012graphlab}, and Parameter Server \cite{li2014scaling},
have been designed to operate on and learn from data at massive scales. Despite the fact that, when compared to a single machine, distributed systems tremendously reduce the storage and (local) computational costs, the inevitable cost of communications across the network can often be the bottleneck of distributed computations. As a result, designing methods which can strike an appropriate balance between the cost of computations and that of communications are increasingly desired.

The desire to reduce communication costs is even more pronounced in the {\it federated learning} framework~\cite{konevcny2016federated,konevcny2016federated2,bonawitz2017practical,mcmahan2017communication,smith2017federated}. 
Similarly to typical settings of distributed computing, federated learning assumes data are distributed over a network across nodes that enjoy reasonable computational resources, e.g., mobile phones, wearable devices, and smart homes. 
However, the network has severely limited bandwidth and high latency. 
As a result, it is imperative to reduce the communications between the center and a node or between two nodes. 
In such settings, the preferred methods are those which can perform expensive local computations with the aim of reducing the overall communications across the network.

Optimization algorithms designed for distributed setting are abundantly found in the literature. 
First-order methods, i.e, those that rely solely on gradient information, are often embarrassingly parallel and easy to implement.
Examples of such methods include distributed variants of stochastic gradient descent (SGD) \cite{mahajan2013parallel,recht2011hogwild,zinkevich2010parallelized},
accelerated SGD \cite{shamir2014distributed}, 
variance reduction SGD \cite{lee2015distributed,reddi2015variance}, stochastic coordinate descent methods~\cite{fercoq2016optimization,liu2015asynchronous,necoara2016parallel,richtarik2016parallel}
and dual coordinate ascent algorithms~\cite{richtarik2016distributed,yang2013trading,zheng2016general}.
The common denominator in all of these methods is that they significantly reduce the amount of local computation. But this blessing comes with an inevitable curse that they, in turn, may require a far greater number of iterations and hence, incur more communications overall. 
Indeed, as a result of their highly iterative nature, many of these first-order methods require several rounds of communications and, potentially, synchronizations in every iteration, and they must do so for many iterations.
In a computer cluster, due to limitations on the network's bandwidth and latency and software system overhead,
communications across the nodes can oftentimes be the critical bottleneck for the distributed optimization.
Such overheads are increasingly exacerbated by the growing number of compute nodes in the network, limiting the scalability of any distributed optimization method that requires many communication-intensive iterations.

To remedy such drawbacks of high number of iterations for distributed optimization, communication-efficient second-order methods, i.e., those that, in addition to the gradient, incorporate curvature information, have also been recently considered~\cite{mahajan2013efficient,shamir2014communication,reddi2016aide,zhang2015disco,jaggi2014communication,ma2015adding,smith2016cocoa}; see also Section~\ref{sec:realted_work}. 
%Among a few, most notably are DANE \cite{shamir2014communication}, AIDE \cite{reddi2016aide}, and DiSCO \cite{zhang2015disco}. Another similar method is CoCoA \cite{jaggi2014communication,ma2015adding,smith2016cocoa}, which is analogous to second-order methods in that it involves sub-problems which are local quadratic approximations to the dual objective function. However, despite the fact that CoCoA makes use of the smoothness condition, it does not exploit any explicit second-order information.
The common feature in all of these methods is that they intend to increase the local computations with the aim of reducing the overall iterations, and hence, lowering the communications. In other words, these methods are designed to perform as much local computation as possible before making any communications across the network.
%In doing so, each worker first locally solves a subproblem using its local data. This is then followed by aggregating (e.g., averaging) the local solutions using one round of communication. 
Pursuing similar objectives, in this paper, we propose a Globally Improved Approximate NewTon (GIANT) method and establish its improved theoretical convergence properties as compared with other similar second-order methods. 
We also showcase the superior empirical performance of GIANT through several numerical experiments. 

The rest of this paper is organized as follows.
Section~\ref{sec:realted_work} briefly reviews prior works most closely related to this paper. Section~\ref{sec:contributions} gives a summary of our main contributions. 
The formal description of the distributed empirical risk minimization problem is given in Section~\ref{sec:problem}, followed by the derivation of various steps of GIANT in Section~\ref{sec:alg}.
Section~\ref{sec:theory} presents the theoretical guarantees.
Section~\ref{sec:exp} shows our large-scale experiments.
The supplementary material provides the proofs.

\subsection{Related Work}
\label{sec:realted_work}

Among the existing distributed second-order optimization methods, the most notably are DANE \cite{shamir2014communication}, AIDE \cite{reddi2016aide}, and DiSCO \cite{zhang2015disco}. Another similar method is CoCoA \cite{jaggi2014communication,ma2015adding,smith2016cocoa}, which is analogous to second-order methods in that it involves sub-problems which are local quadratic approximations to the dual objective function. However, despite the fact that CoCoA makes use of the smoothness condition, it does not exploit any explicit second-order information.

%In this paper, we focus on empirical risk minimization problems with strongly convex and smooth objective function, which is the same to the assumptions made by the prior works DANE, AIDE, DiSCO, and CoCoA. Our method can be potentially extended to non-smooth problem using proximal methods \cite{lee2014proximal} and non-convex problem using the trust-region methods \cite{xu2017newton}.
%The most closely related prior works to our approach here include DANE \cite{shamir2014communication}, AIDE \cite{reddi2016aide}, DiSCO \cite{zhang2015disco}, and CoCoA \cite{jaggi2014communication,ma2015adding,smith2016cocoa}.

We can evaluate the theoretical properties the above-mentioned methods in light of comparison with optimal first-order methods, i.e., accelerated gradient descent (AGD) methods \cite{nesterov1983method,nesterov2013introductory}. 
It is because AGD methods are mostly embarrassingly parallel and can be regarded as the baseline for distributed optimization. Recall that AGD methods, being optimal in worst-case analysis sense~\cite{nemirovskii1983problem}, are guaranteed to convergence to $\EM$-precision in $\OM (\sqrt{\kappa} \log \frac{1}{\EM})$ iterations \cite{nesterov2013introductory}, where $\kappa$ can be thought of as the condition number of the problem. Each iteration of AGD has two rounds of communications---broadcast or aggregation of a vector.

In Table~\ref{tab:communications}, we compare the communication costs with other methods for the ridge regression problem: $\min_{\w } \frac{1}{n} \| \X \w - \y \|_2^2 + \gamma \|\w \|_2^2$.\footnote{As for general convex problems, it is very hard to present the comparison in an easily understanding way. 
	This is why we do not compare the convergence for the general convex optimization.
	}
The communication cost of GIANT has a mere logarithmic dependence on the condition number $\kappa$;
in contrast, the other methods have at least a square root dependence on $\kappa$.
Even if $\kappa$ is assumed to be small, say $\kappa = \OM ( \sqrt{ n} )$, which was made by \cite{zhang2015disco}, GIANT's bound is better than the compared methods regarding the dependence on the number of partitions, $m$.

Our GIANT method is motivated by the subsampled Newton method \cite{roosta2016sub,xu2016sub,pilanci2017newton}.
Later we realized that very similar idea has been proposed by DANE \cite{shamir2014communication}\footnote{GIANT and DANE are identical for quadratic programming; they are different for the general convex problems.}
and FADL~\cite{mahajan2013efficient}, but we show better convergence bounds.
Mahajan \etal \cite{mahajan2013efficient} has conducted comprehensive empirical studies and concluded that the local quadratic approximation, which is very similar to GIANT, is the final method which they recommended.

%Note that $\kappa$ is order $n$ under the optimal choice of regularization parameter.\footnote{For the ridge regression problem , 
%	the optimal choice of $\gamma$ (for balancing the bias and variance) is $\frac{1}{n}$ \cite{wang2017sketched}. 
%	 The condition number of the Hessian matrix $\X^T \X + \I_d$ is $\kappa \approx \frac{ \| \X^T \X \|_2 }{ n \gamma } = \OM (n )$, because the spectral norm $\| \X^T \X \|_2$ grows linearly with $n$.}

\begin{table}[t]\setlength{\tabcolsep}{0.3pt}
	\def\arraystretch{1.2}
	\caption{The number of communications (proportional to the number of iterations) required for the ridge regression problem. 
		Here $n$ is the total number of samples, $d$ is the number of features, $m$ is the number of partitions,
		$\gamma$ is the regularization parameter,
		$\kappa$ is the condition number of the Hessian matrix, $\mu$ is the matrix coherence, $f$ is the objective function, $\w^\star$ is the optimal solution, $\w_t$ is the output after $t$ iterations, and $\tilde{\OM}$ conceals constants (analogous to $\mu$) and logarithmic factors.
	}
	\label{tab:communications}
	\vspace{-1mm}
	\begin{center}
		\begin{footnotesize}
			\begin{tabular}{c c c}
				\hline
				{\bf Method} & \#{\bf Iterations} & {\bf Metric} \\
				\hline
				~~~~GIANT [this work]~~~~ 
				& ~~~~$t={\OM} \Big( \frac{ \log ({ d \kappa } / { \EM } ) }{\log ( { n } / { \mu d m } ) }   \Big)$~~~~
				& ~~$\| \w_t - \w^\star \|_2 \leq \EM $~~ \\
				~~~~DiSCO~\cite{zhang2015disco}~~~~ 
				& ~~~~$t=\tilde{\OM } \Big( \frac{ d \kappa^{1/2 } m^{3/4} }{ n^{3/4} } + \frac{ \kappa^{1/2 } m^{1/4} }{ n^{1/4} } \log \frac{1}{\EM } \Big) $~~~~ 
				& ~~$f (\w_t) - f (\w^\star ) \leq \EM $~~ \\
				~~~~DANE~\cite{shamir2014communication}~~~~  
				& ~~~~$t=\tilde\OM \Big( \frac{\kappa^2 m}{n}  \log \frac{1}{\EM } \Big)$~~~~ 
				& ~~$f (\w_t) - f (\w^\star ) \leq \EM $~~ \\
				~~~~AIDE~\cite{reddi2016aide}~~~~  
				& ~~~~$t=\tilde\OM \Big( \frac{ \kappa^{1/2 } m^{1/4} }{ n^{1/4} } \log \frac{1}{\EM } \Big) $~~~~ 
				& ~~$f (\w_t) - f (\w^\star ) \leq \EM $~~ \\
				~~~~CoCoA~\cite{smith2016cocoa}~~~~  
				& ~~~~$t=\OM \Big( \big(n + \tfrac{1}{\gamma} \big) \log \frac{n}{\EM } \Big) $~~~~ 
				& ~~$f (\w_t) - f (\w^\star ) \leq \EM $~~ \\
				~~~~AGD~~~~ 
				& $t = \OM \Big( {\kappa }^{1/2} \log \frac{d}{ \EM }\Big)$
				& ~~$\| \w_t - \w^\star \|_2 \leq \EM  $~~ \\
				\hline
			\end{tabular}
		\end{footnotesize}
	\end{center}
\end{table}

%The convergence behavior of Newton's method is very well known.
%For quadratic problem, Newton's method converges in one iteration;
%for general strongly convex and smooth problem, in the vicinity of the optimal solution, Newton's method has super-linear convergence and almost free from the condition number.
%In contrast, the second-order methods---DANE, AIDE, and DiSCO---are highly iterative and have high dependence on the condition number (theoretically), which looks unreasonable and does not match their experiments.
%The slow convergence of these second-order methods motivates us to find better methods and theories.

%---------------------------------Figure---------------------------------%
\begin{wrapfigure}{r}{0.52\linewidth}
	\vspace{-20pt}
	%\begin{figure}[!t]
	\begin{center}
		\includegraphics[width=0.48\textwidth]{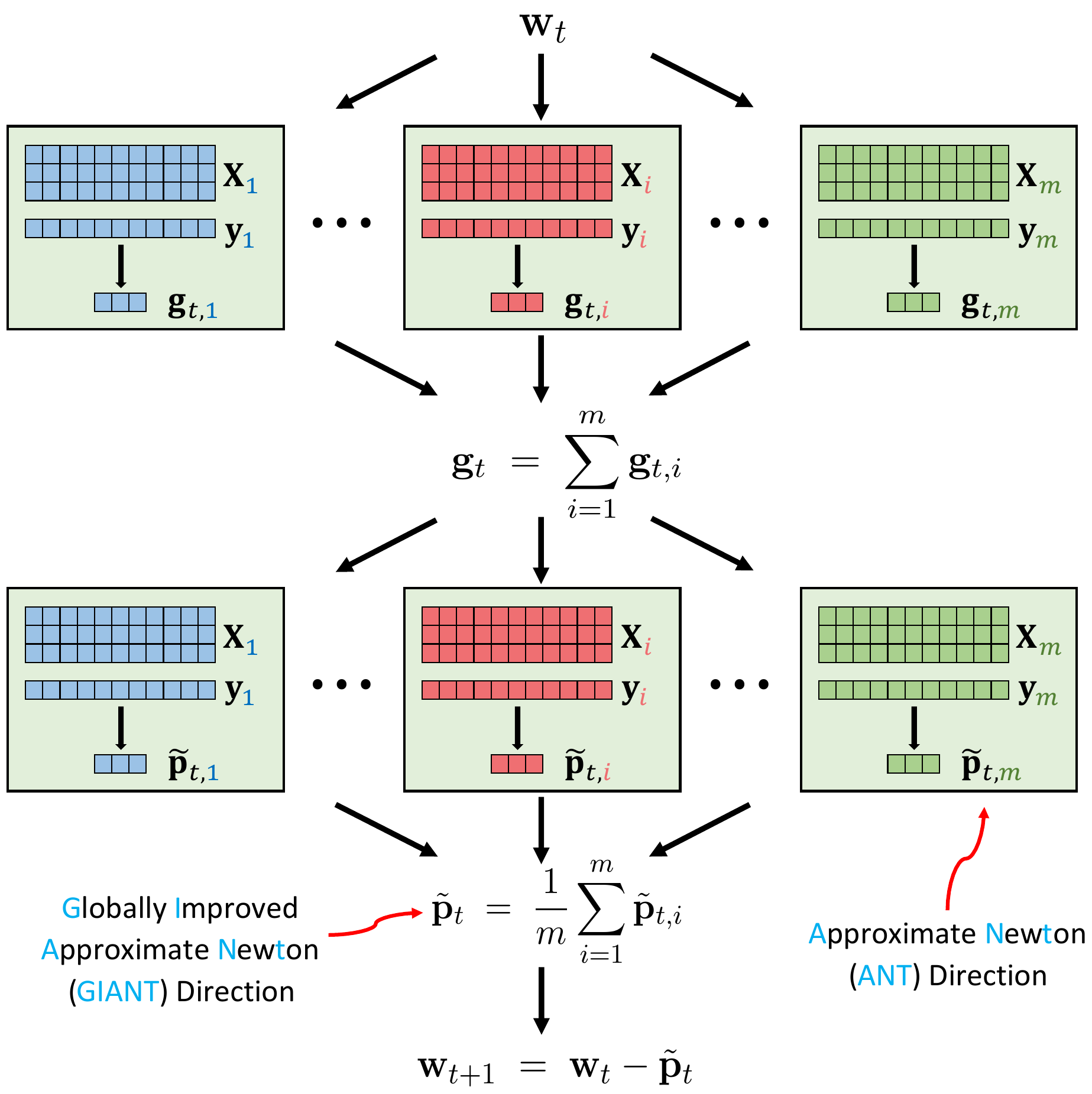}
	\end{center}
	\vspace{-10pt}
	\caption{One iteration of GIANT. 
		Here $\X$ and $\y$ are respectively the features and lables;
		$\X_i$ and $\y_i$ denotes the blocks of $\X$ and $\y$, respectively.
		Each one-to-all operation is a \textsf{Broadcast} and 
		each all-to-one operation is a \textsf{Reduce}.
	}
	\vspace{-2pt}
	\label{fig:algorithm}
	%\end{figure}
\end{wrapfigure}
%---------------------------------Figure---------------------------------%

\subsection{Contributions} 
\label{sec:contributions}

In this paper, we consider the problem of empirical risk minimization involving smooth and strongly convex objective function (which is the same setting considered in prior works of DANE, AIDE, and DiSCO). In this context, we propose a Globally Improved Approximate NewTon (GIANT) method and establish its theoretical and empirical properties as follows.

%\begin{itemize}
%	\item 
	$\bullet$
	For quadratic objectives, we establish global convergence of GIANT.
	To attain a fixed precision, the number of iterations of GIANT (which is proportional to the communication complexity) has a mere \emph{logarithmic} dependence on the condition number.
	In contrast, the prior works have at least square root dependence.
	In fact, for quadratic problems, GIANT and DANE \cite{shamir2014communication} can be shown to be identical. In this light, for such problems, our work improves upon the convergence of DANE.
	
%	\item
$\bullet$
	For more general problems, GIANT has \emph{linear-quadratic convergence} in the vicinity of the optimal solution, which we refer to as ``local convergence''.\footnote{The second-order methods typically have the local convergence issue. Global convergence of GIANT can be trivially established by following \cite{roosta2016global}, however, the convergence rate is not very interesting, as it is worse than the first-order methods.}
	The advantage of GIANT mainly manifests in \emph{big-data} regimes where there are many data points available. 
	In other words, when the number of data points is much larger than the number of features, the theoretical convergence of GIANT enjoys  significant improvement over other similar methods.
	
%	\item 
$\bullet$
	In addition to theoretical features, {GIANT} also exhibits desirable practical advantages. For example, in sharp contrast with many existing distributed Newton-type methods, as well as popular first-order methods, GIANT only involves \emph{one tuning parameter}, i.e., the maximal iterations of its sub-problem solvers, which makes GIANT easy to implement in practice.
	Furthermore, our experiments on a computer cluster show that GIANT consistently outperforms AGD, L-BFGS, and DANE.
%\end{itemize}

%%%%%%%%%%%%%%%%%%%%%%%%%%%%%%%%%%%%%%%%%%%%%%%%%%%%%%%%%%%%%%%%%%%%%%%%%%%%%%
%%%%%%%%%%%%%%%%%%%%%%%%%%%%%%%%%%%%%%%%%%%%%%%%%%%%%%%%%%%%%%%%%%%%%%%%%%%%%%
%%%%%%%%%%%%%%%%%%%%%%%%%%%%%%%%%%%%%%%%%%%%%%%%%%%%%%%%%%%%%%%%%%%%%%%%%%%%%%
%%%%%%%%%%%%%%%%%%%%%%%%%%%%%%%%%%%%%%%%%%%%%%%%%%%%%%%%%%%%%%%%%%%%%%%%%%%%%%

\section{Problem Formulation} \label{sec:problem}

In this paper, we consider the distributed variant of empirical risk minimization, a supervised-learning problem arising very often in machine learning and data analysis~\cite{shalev2014understanding}. More specifically, let $\x_1 , \cdots , \x_n \in \RB^d$ be the input feature vectors and $y_1 , \cdots , y_n \in \RB$
be the corresponding response.
The goal of supervised learning is to compute a model from the training data,
which can be achieved by minimizing an empirical risk function, i.e., 
\begin{small}
\begin{eqnarray} \label{eq:problem}
	\min_{\w \in \RB^d } \;
	\bigg\{  f (\w )
	\; \triangleq \; 
	\frac{1}{n} \sum_{j=1}^n  \ell_j (\w^T \x_j) + \frac{\gamma }{2} \| \w \|_2^2 \bigg\} ,
\end{eqnarray}
\end{small}%
where $\ell_j : \RB \mapsto \RB$ is convex, twice differentiable, and smooth.
We further assume that $ f $ is strongly convex, which in turn, implies the uniqueness of the minimizer of~\eqref{eq:problem}, denoted throughout the text by $ \w^\star $.
Note that $y_j$ is implicitly captured by $\ell_j$. 
Examples of the loss function, $ \ell_j $, appearing in~\eqref{eq:problem} include 
\begin{align*}
	&\textrm{linear regression:} \qquad
	\ell_j (z_j) = \tfrac{1}{2} (z_j - y_j)^2 ,\\
	&\textrm{logistic regression:} \quad \;
	\ell_j (z_j) = \log ( 1 + e^{-z_j y_j } ) .
\end{align*}

%\subsection{Data Partitioning}

Suppose the $n$ feature vectors and loss functions $ (\x_1, \ell_1) , \cdots , (\x_n, \ell_n) $ are partitioned among $m$ worker machines.
Let $s \triangleq  {n}/{m}$ be the local sample size.
Our theories require $s > d$; nevertheless, GIANT empirically works well for $s < d$.

We consider solving \eqref{eq:problem} in the regimes where $n \gg d$. 
We assume that the data points, $\{\x_i\}_{i=1}^{n}$ are partitioned among $m$ machines, with possible overlaps, such that the number of local data is larger than $d$.
Otherwise, if $n \ll d$, we can consider the dual problem and partition features.
If the dual problem is also decomposable, smooth, strongly convex, and unconstrained, e.g., ridge regression, then our approach directly applies.

%%%%%%%%%%%%%%%%%%%%%%%%%%%%%%%%%%%%%%%%%%%%%%%%%%%%%%%%%%%%%%%%%%%%%%%%%%%%%%
%%%%%%%%%%%%%%%%%%%%%%%%%%%%%%%%%%%%%%%%%%%%%%%%%%%%%%%%%%%%%%%%%%%%%%%%%%%%%%
%%%%%%%%%%%%%%%%%%%%%%%%%%%%%%%%%%%%%%%%%%%%%%%%%%%%%%%%%%%%%%%%%%%%%%%%%%%%%%
%%%%%%%%%%%%%%%%%%%%%%%%%%%%%%%%%%%%%%%%%%%%%%%%%%%%%%%%%%%%%%%%%%%%%%%%%%%%%%

\section{Algorithm Description} \label{sec:alg}

In this section, we present the algorithm derivation and complexity analysis. 
GIANT is a centralized and synchronous method; one iteration of GIANT is depicted in Figure~\ref{fig:algorithm}.
The key idea of GIANT is avoiding forming of the exact Hessian matrices $\H_t \in \RB^{d\times d}$ in order to avoid expensive communications.

\subsection{Gradient and Hessian}

GIANT iterations require the exact gradient, which in the $t$-th iteration, can be written as 
%\vspace{-1mm}
\begin{small} 
\begin{eqnarray} \label{eq:gradient}
	\g_t \; = \; \nabla f(\w_t) 
	\; = \; \frac{1}{n} \sum_{j=1}^n  \ell_j' (\w_t^T \x_j ) \; \x_j + \gamma \w_t 
	\; \in \; \RB^d .
\end{eqnarray}
\end{small}%
The gradient, $ \g_t $ can be computed, embarrassingly, in parallel. 
The driver \textsf{Broadcasts} $\w_t$ to all the worker machines. Each machine then uses its own $\{ (\x_j, \ell_j) \}$ to compute its local gradient. Subsequently, the driver performs a \textsf{Reduce} operation to sum up the local gradients and get $\g_t$. 
The per-iteration communication complexity is $\tilde{\OM } (d )$ words, where $\tilde{\OM } $ hides the dependence on $m$ (which can be $m$ or $\log m$, depending on the network structure).

More specifically, in the $t$-th iteration, the Hessian matrix at $\w_t \in \RB^d$ can be written as
%\vspace{-1mm}
\begin{small} 
\begin{eqnarray} \label{eq:hessian}
	\H_t \; = \; \nabla^2 f(\w_t) 
	\; = \; \frac{1}{n} \sum_{j=1}^n  \ell_j'' (\w_t^T \x_j ) \cdot \x_j \x_j^T + \gamma \I_d .
\end{eqnarray}
\end{small}%
To compute the exact Hessian, the driver must aggregate the $m$ local Hessian matrices (each of size $d\times d$) by one \textsf{Reduce} operation, which has $\tilde{\OM }  (d^2 )$ communication complexity and is obviously impractical when $d$ is thousands.
The Hessian approximation developed in this paper has a mere $\tilde{\OM }  (d )$ communication complexity which is the same to the first-order methods.

\subsection{Approximate NewTon (ANT) Directions}

Assume each worker machine locally holds $s$ random samples drawn from $\{ (\x_j , \ell_j) \}_{j=1}^n$.\footnote{If the samples themselves are i.i.d.\ drawn from some distribution, then a data-independent partition is equivalent to uniform sampling. Otherwise, the system can \textsf{Shuffle} the data.}
Let $\JM_i$ be the set containing the indices of the samples held by the $i$-th machine, and $ s = |\JM_i | $ denote its size.
%; the sample size $s = |\JM |$ can be larger or smaller than or equal to $\tfrac{n}{m}$.
Each worker machine can use its local samples to form a local Hessian matrix 
\begin{eqnarray*} 
	\widetilde{\H}_{t,i}
	\; = \; \frac{1}{s} \sum_{j \in \JM_i }  \ell_j'' (\w_t^T \x_j ) \cdot \x_j \x_j^T + \gamma \I_d  .
\end{eqnarray*}
Clearly, $\EB [ \widetilde{\H}_{t,i} ] = \H_t$.
We define the Approximate NewTon ({ANT}) direction by $ \tilde{\pp}_{t,i}
\; = \; \widetilde{\H}_{t,i}^{-1} \g_t$.
The cost of computing the ANT direction $\tilde{\pp}_{t,i}$ in this way, involves $\OM (s d^2)$ time to form the $d\times d$ dense matrix $\widetilde{\H}_{t,i}$ and $\OM (d^3)$ to invert it. 
To reduce the computational cost, we opt to compute the ANT direction by the conjugate gradient (CG) method~\cite{nocedal2006numerical}.
Let $\a_{j} = \sqrt{ \ell_j'' (\w_t^T \x_j ) } \cdot \x_j \in \RB^d$,
\begin{eqnarray} \label{eq:def:at}
	\A_t 
	\; = \; [\a_{1}^T ; \cdots ; \a_{n}^T ] 
	\; \in \; \RB^{n\times d} ,
\end{eqnarray}
and $\A_{t,i} \in \RB^{s\times d}$ contain the rows of $\A_t$ indexed by the set $\JM_i$. 
Using the matrix notation, we can write the local Hessian matrix as
\begin{eqnarray} \label{eq:def_local_hessian}
	\widetilde{\H}_{t,i}
	\; = \; \tfrac{1}{s} \A_{t,i}^T \A_{t,i} + \gamma \I_d  .
\end{eqnarray}
Employing CG, it is thus unnecessary to explicitly form $\widetilde{\H}_{t,i}$. Indeed, one can simply approximately solve
\begin{eqnarray}\label{eq:ANT}
	\big( \tfrac{1}{s} \A_{t,i}^T \A_{t,i} + \gamma \I_d  \big) \, \pp
	\; = \; \g_t
\end{eqnarray}
in a ``Hessian-free'' manner, i.e., by employing only Hessian-vector products in CG iterations.

\subsection{Globally Improved ANT (GIANT) Direction}

Using random matrix concentration, we can show that for sufficiently large $s$, the local Hessian matrix $\widetilde{\H}_{t,i}$ is a spectral approximation to $\H_t$. Now let $\tilde{\pp}_{t,i}$ be an ANT direction.
The Globally Improved ANT (GIANT) direction is defined as
\begin{eqnarray}
	\tilde{\pp}_{t}
	\; = \; \frac{1}{m} \sum_{i=1}^m \tilde{\pp}_{t,i}
	\; = \; \frac{1}{m} \sum_{i=1}^m \widetilde{\H}_{t,i}^{-1} \g_t 
	\; = \; \widetilde{\H}_t^{-1} \g_t .
\end{eqnarray}
Interestingly, here $ \widetilde{\H}_t$ is the \emph{harmonic mean} defined as 
$\widetilde{\H}_t \triangleq ( \frac{1}{m} \sum_{i=1}^m \widetilde{\H}_{t,i}^{-1} )^{-1}$,
whereas the true Hessian $\H_t$ is the \emph{arithmetic mean} defined as $\H_t \triangleq \frac{1}{m} \sum_{i=1}^m \widetilde{\H}_{t,i}$.
%If all the local Hessian matrices $\widetilde{\H}_{t,1} , \cdots , \widetilde{\H}_{t,m}$ approximate the true Hessian matrix,
If the data is incoherent, that is, the ``information'' is spread-out rather than concentrated to a small fraction of samples,
then the harmonic mean and the arithmetic mean are very close to each other,
and thereby the GIANT direction $\tilde{\pp}_{t} = \widetilde{\H}^{-1} \g_t$ very well approximates the true Newton direction $\H^{-1} \g_t$.
This is the intuition of our global improvement.

The motivation of using the harmonic mean, $ \widetilde{\H}_t $, to approximate the arithmetic mean (the true Hessian matrix), $\H_t$, is the communication cost.
Computing the arithmetic mean $\H_t \triangleq \frac{1}{m} \sum_{i=1}^m \widetilde{\H}_{t,i}$ would require the communication of $d\times d$ matrices which is very expensive.
In contrast, computing $\tilde{\pp}_{t}$ merely requires the communication of $d$-dimensional vectors.

\subsection{Time and Communication Complexities}

For each worker machine, the per-iteration time complexity is $\OM ( s d q )$, where $s$ is the local sample size and $q$ is the number of CG iterations for (approximately) solving \eqref{eq:ANT}.
(See Proposition~\ref{cor:cg_q} for the setting of $q$.)
If the feature matrix $\X \in \RB^{n\times d}$ has a sparsity of $\varrho = {\nnz (\X)}/{(n d)} < 1$, the expected per-iteration time complexity is then $\OM (\varrho s d q )$. 

Each iteration of GIANT has four rounds of communications: two \textsf{Broadcast} for sending and two \textsf{Reduce} for aggregating some $d$-dimensional vector.
If the communication is in a tree fashion, the per-iteration communication complexity is then $\tilde{\OM }  (d )$ words, where $\tilde{\OM} $ hides the factor involving $m$ which can be $m$ or $\log m$.

%%%%%%%%%%%%%%%%%%%%%%%%%%%%%%%%%%%%%%%%%%%%%%%%%%%%%%%%%%%%%%%%%%%%%%%%%%%%%%
%%%%%%%%%%%%%%%%%%%%%%%%%%%%%%%%%%%%%%%%%%%%%%%%%%%%%%%%%%%%%%%%%%%%%%%%%%%%%%
%%%%%%%%%%%%%%%%%%%%%%%%%%%%%%%%%%%%%%%%%%%%%%%%%%%%%%%%%%%%%%%%%%%%%%%%%%%%%%
%%%%%%%%%%%%%%%%%%%%%%%%%%%%%%%%%%%%%%%%%%%%%%%%%%%%%%%%%%%%%%%%%%%%%%%%%%%%%%

\section{Theoretical Analysis} \label{sec:theory}

In this section, we formally present the convergence guarantees of GIANT. Section~\ref{sec:converge:quadratic} focuses on quadratic loss and treats the global convergence of GIANT. This is then followed by local convergence properties of GIANT for more general non-quadratic loss in Section~\ref{sec:converge:general}.
For the results of Sections~\ref{sec:converge:quadratic} and \ref{sec:converge:general}, we require that the local linear system to obtain the local Newton direction is solved exactly. Section~\ref{sec:converge:inexact} then relaxes this requirement to allow for inexactness in the solution, and establishes similar convergence rates as those of exact variants. 

For our analysis here, we frequently make use of the notion of \emph{matrix row coherence}, defined as follows.

\begin{definition} [Coherence]
	\label{def:coherence}
	Let $\A \in \RB^{n\times d}$ be any matrix and $\U \in \RB^{n\times d}$ be its column orthonormal bases.
	The {\it row coherence} of $\A$ is $\mu (\A ) = \frac{n}{d} \max_{j } \|\u_{j}\|_2^2 \in [1, \frac{n}{d}]$. 
\end{definition}

\begin{remark} \label{remark:coherence}
	Our work assumes $\A_t \in \RB^{n\times d} $, which is defined in \eqref{eq:def:at}, is incoherent, namely $\mu (\A_t )$ is small.
	The prior works, DANE, AIDE, and DiSCO, did not use the notation of \emph{incoherence};
	instead, they assume $ \nabla_{\w }^2 l_j (\w^T \x_j ) \, |_{\w = \w_t} = \a_j \a_j^T$ is upper bounded for all $j \in [n]$ and $\w_t  \in \RB^d$, 
	where $\a_{j} \in \RB^d$ is the $j$-th row of $\A_t$.
	Such an assumption is different from but has similar implication as our incoherence assumption;
	under either of the two assumptions, it can be shown that the Hessian matrix can be approximated using a subset of samples selected uniformly at random.
\end{remark}

\subsection{Quadratic Loss} \label{sec:converge:quadratic}
In this section, we consider a special case of \eqref{eq:problem} with $\ell_i (z) = (z - y_i)^2/2$, i.e., the quadratic optimization problems:
\begin{small}
\begin{eqnarray} \label{eq:linear}
	f (\w )
	& = & \tfrac{1}{2n} \big\| \X \w - \y \big\|_2^2 + \tfrac{\gamma}{2}  \| \w \|_2^2.
	%& = & \frac{1}{2n} \sum_{i=1}^n \big( \w^T \x_i - y_i \big)^2 + \frac{\gamma }{2} \w^T  ,
\end{eqnarray}
\end{small}%
The Hessian matrix is given as $\nabla^2 f (\w) = \frac{1}{n} \X^T \X + \gamma \I_d$, which does not depend on $\w$. Theorem~\ref{thm:quadratic} describes the convergence of the error in the iterates, i.e., $\De_t \triangleq \w_t - \w^\star$.

\begin{theorem} \label{thm:quadratic}
	Let $\mu$ be the row coherence of $\X \in \RB^{n\times d}$ and $m$ be the number of partitions.
	Assume the local sample size satisfies $s \geq \frac{3 \mu  d }{ {\textcolor{OliveGreen}{\eta}}^{2} } \log \frac{m d}{\delta} $ for some ${\textcolor{OliveGreen}{\eta}} , \delta \in (0, 1)$.
	It holds with probability $1-\delta$ that
	\begin{small}
	\begin{align*}
		\| \De_t \|_2 \leq \alpha^t \, \sqrt{ \kappa   } \, \| \De_0 \|_2,
	\end{align*}
\end{small}%
	where 
	$\alpha = \frac{\textcolor{OliveGreen}{\eta}}{\sqrt{ m } }  + {\textcolor{OliveGreen}{\eta}}^2 $
	and $\kappa$ is the condition number of $\nabla^2 f(\w) = \frac{1}{n} \X^T \X + \gamma \I_d$.
\end{theorem}

%The role of matrix coherence $ \mu $ and the sample size $ s $ can be observed from Theorem~\ref{thm:quadratic}. More specifically, if $\X$'s coherence, $ \mu $, is small and $s$ is large enough compared to $d$,
%then we obtain a linear convergence rate in the error, $\|\De_t \|_2$.

\begin{remark}
	The theorem can be interpreted in the this way.
	Assume the total number of samples, $n$, is at least $ 3 \mu d m \log (m d)$.
	Then
	\begin{small}
	\begin{align*}
	\| \De_t \|_2 \; \leq \; 
	\Big( \tfrac{ 3 \mu d m \log (md / \delta ) }{ n } 
	+ \sqrt{\tfrac{ 3 \mu d \log (md / \delta ) }{ n }} \, \Big)^t \,
	\sqrt{ \kappa   } \, \| \De_0 \|_2 
	\end{align*}
\end{small}%
	holds with probability at least $1 - \delta$.
	
	If the total number of samples, $n$, is substantially bigger than $\mu d m$, then GIANT converges in a very small number of iterations.
	Furthermore, to reach a fixed precision, say $\| \De_t \|_2 \leq \EM$, the number of iterations, $t$, has a mere logarithmic dependence on the condition number, $\kappa$.
\end{remark}

\subsection{General Smooth Loss} \label{sec:converge:general}

For more general (not necessarily quadratic) but smooth loss, GIANT has linear-quadratic local convergence,
which is formally stated in Theorem~\ref{thm:general} and Corollary~\ref{cor:general}.
Let $\H^\star = \nabla^2 f (\w^\star) $ and $\H_t = \nabla^2 f (\w_t)$.
For this general case, we assume the Hessian is $L$-Lipschitz, which is a standard assumption in analyzing second-order methods.

\begin{assumption} \label{assumption:lipschitz}
	The Hessian matrix is $L$-Lipschitz continuous, i.e., $\big\| \nabla^2 f (\w) - \nabla^2 f (\w') \big\|_2 \leq L\|\w - \w' \|_2$, for all $ \w$ and $\w'$.
\end{assumption}

Theorem~\ref{thm:general} establishes the linear-quadratic convergence of $\De_t \triangleq \w_t - \w^\star$.
We remind that $\A_t \in \RB^{n\times d}$ is defined in \eqref{eq:def:at} (thus $\A_t^T \A_t + \gamma \I_d = \H_t$).
Note that, unlike Section~\ref{sec:converge:quadratic}, the coherence of $\A_t$, denote $\mu_t$, changes with iterations.

\begin{theorem} \label{thm:general}
	Let $\mu_t \in [1, {n}/{d}]$ be the coherence of $\A_t $ and $m$ be the number of partitions.
	Assume the local sample size satisfies $s_t \geq \frac{3 \mu_t  d }{ {\textcolor{OliveGreen}{\eta}}^{2} } \log \frac{m d}{\delta} $ for some ${\textcolor{OliveGreen}{\eta}} , \delta \in (0, 1)$.
	Under Assumption~\ref{assumption:lipschitz},
	it holds with probability $1-\delta$ that
	\begin{small}
	\begin{align*}
		\big\| \De_{t+1} \big\|_2 
		\; \leq \; \max \Big\{ \alpha \,
		\sqrt{ \tfrac{\sigma_{\max } (\H_t )}{\sigma_{\min} (\H_t )} }
		\big\| \De_t \big\|_2,  \tfrac{2L}{\sigma_{\min} (\H_t )} \big\| \De_t \big\|_2^2 \Big\} ,
	\end{align*}
\end{small}%
	where $\alpha =  \frac{\textcolor{OliveGreen}{\eta}}{\sqrt{ m }} + {\textcolor{OliveGreen}{\eta}}^2 $.
\end{theorem}

\begin{remark}
	The standard Newton's method is well known to have local quadratic convergence;
	the quadratic term in Theorem~\ref{thm:general} is the same as Newton's method.
	The quadratic term is caused by the non-quadritic objective function.
	The linear term arises from the Hessian approximation.
	For large sample size, $s$, equivalently, small $\textcolor{OliveGreen}{\eta}$, the linear term is small.
\end{remark}

Note that in Theorem~\ref{thm:general} the convergence depends on the condition numbers of the Hessian at every point.
Due to the Lipschitz assumption on the Hessian, it is easy to see that 
the condition number of the Hessian in a neighborhood of $\w^\star$
is close to $\kappa (\H^\star)$.
This simple observation implies Corollary~\ref{cor:general}, in which the dependence of the local convergence of GIANT on iterations via $ \H_t $ is removed. 
%For this, we need to make the following assumption which describes a neighborhood radius around $ \w^{\star} $, in which such local iteration independent local convergence is obtained.
%(Although $\vartheta$ depends on $\H_t$, it can be removed because $\vartheta \leq 1$.)

\begin{assumption} \label{assumption:close}
	Assume $\w_t$ is close to $\w^\star$ in that
	$ \| \De_t \|_2 \leq 3 L \cdot \sigma_{\min} (\H^\star) $,
	where $L$ is defined in Assumption~\ref{assumption:lipschitz}.
\end{assumption}

\begin{corollary}\label{cor:general}
	Under the same setting as Theorem~\ref{thm:general} and Assumption \ref{assumption:close}, 
	it holds with probability $1-\delta$ that
	\begin{small}
	\begin{align*}
		\big\| \De_{t+1} \big\|_2 
		\; \leq \; \max \Big\{ 2 \alpha \sqrt{\kappa } \,
		\big\| \De_t \big\|_2, 
		\tfrac{3L}{\sigma_{\min} (\H^\star )} \big\| \De_t \big\|_2^2 \Big\} ,
	\end{align*}
\end{small}%
	where $\kappa$ is the condition number of the Hessian matrix at $\w^\star$.
\end{corollary}

\subsection{Inexact Solutions to Local Sub-Problems} \label{sec:converge:inexact}

In the $t$-th iteration, the $i$-th worker locally computes $\tilde{\pp}_{t,i}$ by solving 
$\widetilde{\H}_{t,i} \pp = \g_t$, where $\widetilde{\H}_{t,i}$ is the $i$-th local Hessian matrix defined in \eqref{eq:def_local_hessian}.
In high-dimensional problems, say $d \geq 10^4$, the exact formation of $\widetilde{\H}_{t,i} \in \RB^{d\times d}$ and its inversion are impractical. 
Instead, we could employ iterative linear system solvers, such as CG, to inexactly solve the arising linear system in \eqref{eq:ANT}.
Let $\tilde{\pp}_{t,i}' $ be an inexact solution which is close to $\tilde{\pp}_{t,i} \triangleq \widetilde{\H}_{t,i}^{-1} \g_t$, in the sense that
\begin{small}
\begin{align} \label{eq:inexact_assumption}
	\Big\| \widetilde{\H}_{t,i}^{1/2} \,
	\big( \tilde{\pp}_{t,i}' - \tilde{\pp}_{t,i} \big) \Big\|_2
	\; \leq \; \frac{{\epsilon}_0}{2} 
	\Big\| \widetilde{\H}_{t,i}^{1/2} \, \tilde{\pp}_{t,i}  \Big\|_2 ,
\end{align}
\end{small}%
for some ${{\epsilon}}_0 \in (0, 1)$. GIANT then takes $\tilde{\pp}_t' = \frac{1}{m} \sum_{i=1}^m \tilde{\pp}_{t,i}'$
as the approximate Newton direction in lieu of $\tilde{\pp}_t$. 
In this case, as long as ${{\epsilon}}_0$ is of the same order as $\frac{\textcolor{OliveGreen}{\eta}}{\sqrt{ m }} + {\textcolor{OliveGreen}{\eta}}^2$,
the convergence rate of such inexact variant of GIANT remains similar to the exact algorithm in which the local linear system is solved exactly. Theorem~\ref{thm:inexact} makes convergence properties of inexact GIANT more explicit.

\begin{theorem} \label{thm:inexact} 
	Suppose inexact local solution to \eqref{eq:ANT}, denote $\tilde{\pp}_{t,i}'$, satisfies \eqref{eq:inexact_assumption}.
	Then Theorems~\ref{thm:quadratic} and \ref{thm:general} and Corollary~\ref{cor:general} all continue to hold with 
	$\alpha =  \big( \frac{{\textcolor{OliveGreen}{\eta}} }{ \sqrt{ m }}  + {\textcolor{OliveGreen}{\eta}}^2 \big) + \epsilon_0$.
\end{theorem}

Proposition~\ref{cor:cg_q} gives conditions to guarantee \eqref{eq:inexact_assumption}, which is, in turn, required for Theorem~\ref{thm:inexact}. 

\begin{proposition} \label{cor:cg_q}
	To compute an inexact local Newton direction from the sub-problem~\eqref{eq:ANT}, suppose 
	each worker performs
	%\vspace{-1mm}
	\begin{small}
	\[
	q \; = \;
	\log \tfrac{ 8}{\epsilon_0^2} 
	\; \big/ \; 
	\log \tfrac{ \sqrt{{\tilde\kappa}_t} + 1 }{\sqrt{{\tilde\kappa}_t} - 1} 
	\; \approx \; \tfrac{ \sqrt{\kappa_t } - 1}{ 2 } \log \tfrac{ 8 }{\epsilon_0^2}
	\]
\end{small}%
	iterations of CG, initialized at zero, where ${\tilde\kappa}_t$ and $\kappa_t$ are, respectiely, the condition number of $\widetilde\H_{t,i}$ and $\H_t$. Then requirement \eqref{eq:inexact_assumption} is satisfied.
\end{proposition}

\section{Experiments} \label{sec:exp}

Our experiments are conducted on logistic regression with $\ell_2$ regularization, i.e.,
\vspace{-2mm}
\begin{align} 
\label{eq:def:logis}
\min_{\w} \frac{1}{n} \sum_{j=1}^n \log \big( 1 + \exp ({-y_j \x_j^T \w }) \big)
+ \frac{\gamma }{2} \| \w \|_2^2,
\end{align}
where $\x_j \in \RB^d$ is a feature vector
and $y_j \in \{-1 , +1 \}$ is the corresponding response.
For an unseen test sample $\x'$, the logistic regression makes prediction by $y' = \sgn (\w^T \x')$.

All the compared methods are implemented in Scala and Apache Spark \cite{zaharia2010spark} and experiments are conducted on the Cori Supercomputer maintained by NERSC. Cori is a Cray XC40 system with 1632 compute nodes, each of which has two 2.3GHz 16-core Haswell processors and 128GB of DRAM. The Cray Aries high-speed interconnect linking the compute nodes is configured in a dragonfly topology.

\vspace{-1mm}
\subsection{Compared Methods}

We compare GIANT with three methods: Accelerated Gradient Descent (AGD)~\cite{nesterov2013introductory}, 
Limited memory BFGS (L-BFGS)~\cite{liu1989limited}, and Distributed Approximate NEwton (DANE) \cite{shamir2014communication}.
For each method, we try different settings for their respective parameters and report the best performance.

\begin{itemize}[leftmargin=*,wide=0em, itemsep=-2pt,topsep=-2pt, label = {\bfseries -}]
	\item 
	\textbf{GIANT} has only one tuning parameter, i.e., the maximum number of CG iterations to approximately solve the local sub-problems.
	\item 
	\textbf{Accelerated Gradient Descent (AGD)}
	repeats the following two steps:
	$\v_{t+1} = \beta \v_t + \g_t$ and $\w_{t+1} = \w_t - \alpha \v_{t+1}$.
	Here $\g_t$ is the gradient, $\v_t$ is the momentum, and $\alpha > 0$ and $\beta \in [0, 1)$ are tuning parameters.
	We choose $\alpha$ and $\beta$ from $\{0.1, 1, 10, 100\}$ and $\{0.5, 0.9, 0.95, 0.99 , 0.999 \}$, respectively.
	We are aware of several variants of AGD; we just compare with one of them.
	\item 
	\textbf{L-BFGS} is a quasi-Newton method that approximates the BFGS method using a limited amount of memory.
	L-BFGS has one tuning parameter, i.e., the history size, which introduces a trade-off between the memory and computational costs as well as convergence rate.
	\item 
	\textbf{Distributed Approximate NEwton (DANE)}
	was proposed by \cite{shamir2014communication}; here we use the inexact DANE studied by \cite{reddi2016aide}.
	DANE bears a strong resemblance with GIANT:
	the local sub-problem of DANE is still a logistic regression,
	whereas the sub-problem of GIANT is a quadratic approximation to the logistic regression.
	We have tried AGD and SVRG \cite{johnson2013accelerating} to solve the corresponding sub-problem and found SVRG to perform much better than AGD.
	DANE seems to be sensitive to two parameters:
	the step size (learning rate) and the stopping criterion of SVRG.
	We choose the step size and the maximal iteration of SVRG from $\{0.1, 1, 10, 100\}$ and $\{30, 100, 300\}$, respectively.
\end{itemize}

%---------------------------------Figure---------------------------------%
\begin{figure*}[!t]
	\begin{center}
		\subfigure[Objective function value $f(\w_t) - f(\w^\star)$ against the wall-clock time.]{\includegraphics[width=0.9\textwidth]{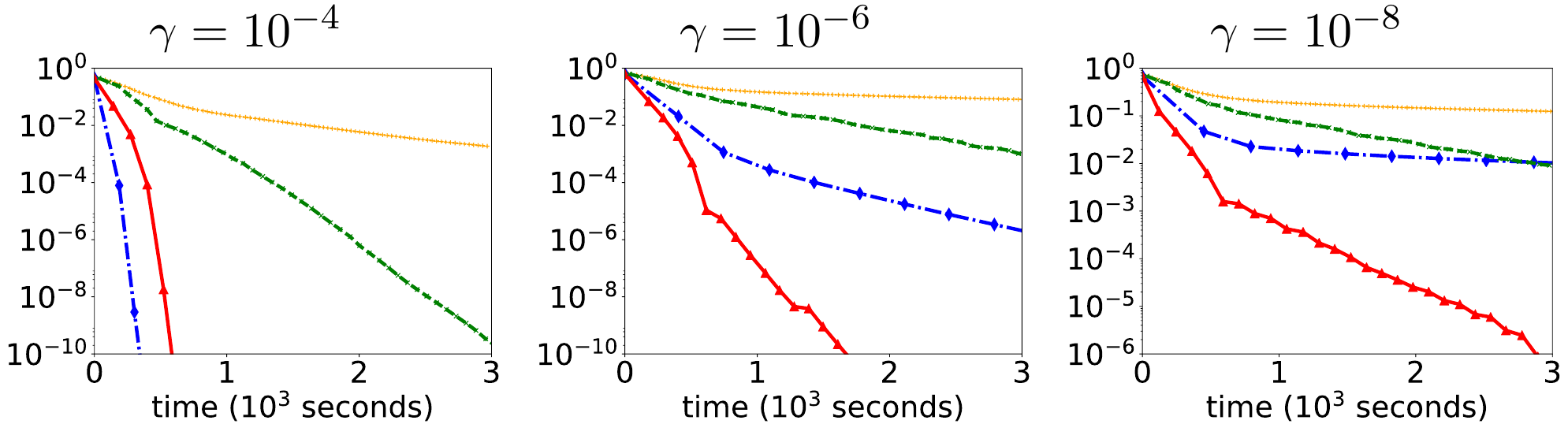}}
		\subfigure[Classification error rate (\%) on the test set against the wall-clock time.]{\includegraphics[width=0.9\textwidth]{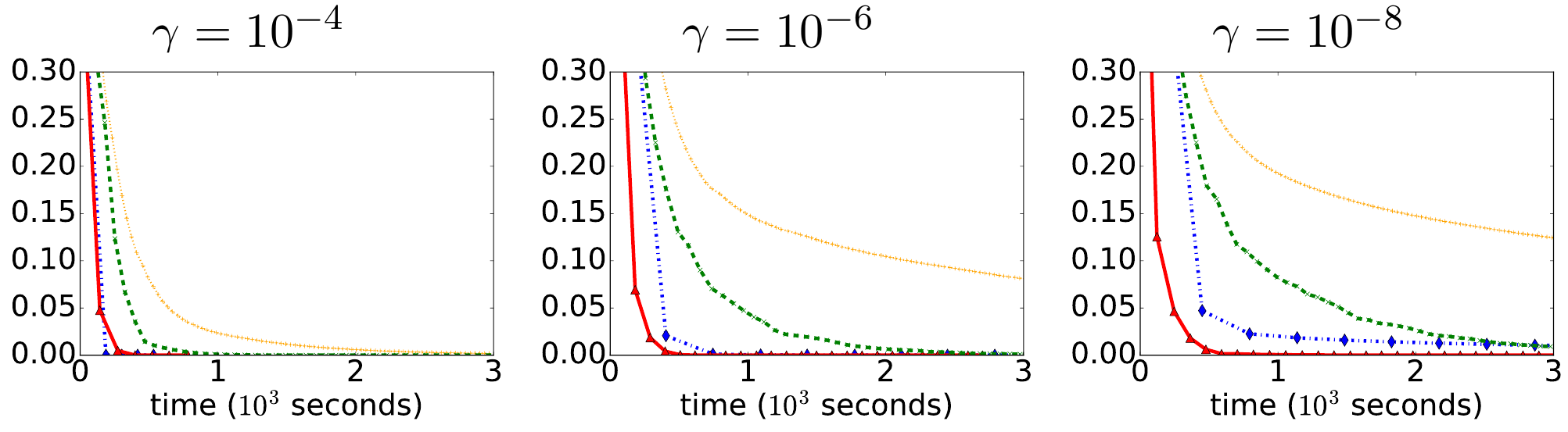}}
		\includegraphics[width=0.5\textwidth]{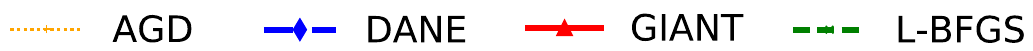}
	\end{center}
	\caption{Results on the MNIST8K data with random feature mapping ($n=8M$ and $d=10K$).
		The feature matrix is dense.
		The regularization parameter $\gamma$ is defined in~\eqref{eq:def:logis}.
		We uses $15$ compute nodes (totally $480$ cores) and partition the data to $m=89$ parts.
	}
	\label{fig:mnist}
	\vspace{-2mm}
\end{figure*}
%---------------------------------Figure---------------------------------%

%---------------------------------Figure---------------------------------%
\begin{figure*}[!t]
	\begin{center}
		\subfigure[Objective function value $f(\w_t) - f(\w^\star)$ against the wall-clock time.]{\includegraphics[width=0.9\textwidth]{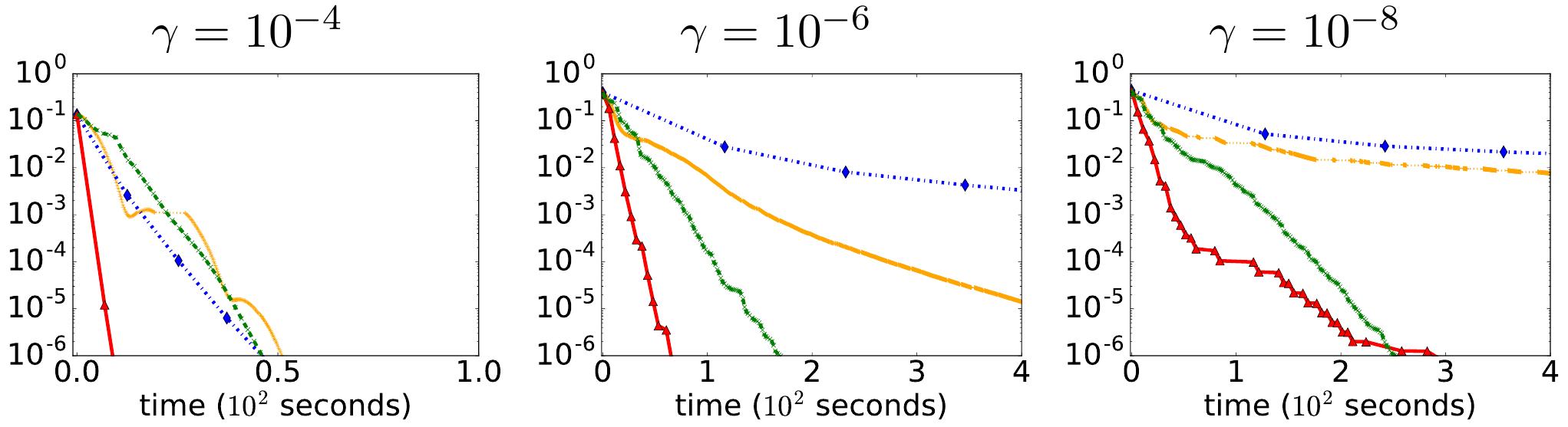}}
		\subfigure[Classification error rate (\%) on the test set against the wall-clock time.]{\includegraphics[width=0.9\textwidth]{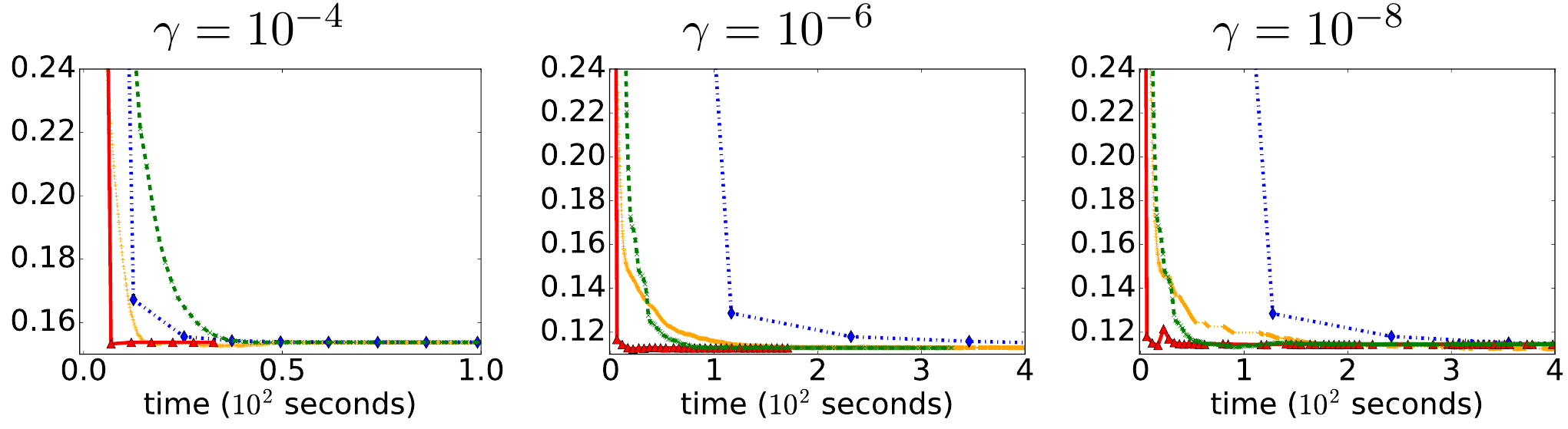}}
		\includegraphics[width=0.5\textwidth]{figure/spark/legend.pdf}
	\end{center}
	\caption{Results on the Epsilon data with random feature mapping ($n=500K$ and $d=10K$).
		The feature matrix is dense.
		The regularization parameter $\gamma$ is defined in~\eqref{eq:def:logis}.
		We uses $15$ compute nodes (totally $480$ cores) and partition the data to $m=89$ parts.
	}
	\vspace{-2mm}
	\label{fig:epsilon}
\end{figure*}
%---------------------------------Figure---------------------------------%

%---------------------------------Figure---------------------------------%
\begin{figure*}[!t]
	\begin{center}
		\subfigure[Objective function value $f(\w_t) - f(\w^\star)$ against the wall-clock time.]{\includegraphics[width=0.9\textwidth]{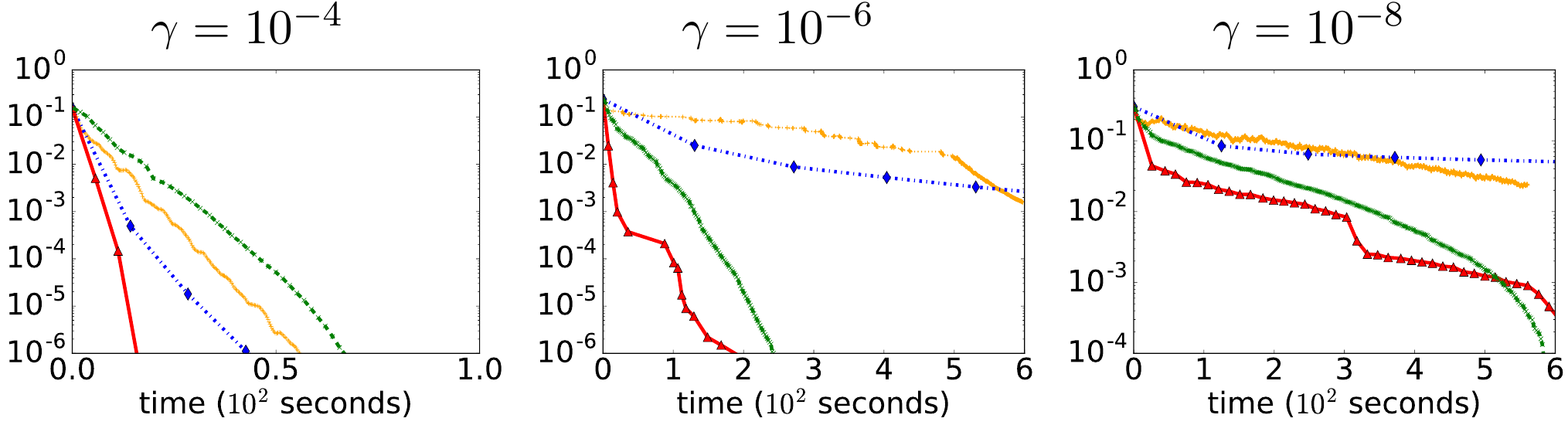}}
		\subfigure[Classification error rate (\%) on the test set against the wall-clock time.]{\includegraphics[width=0.9\textwidth]{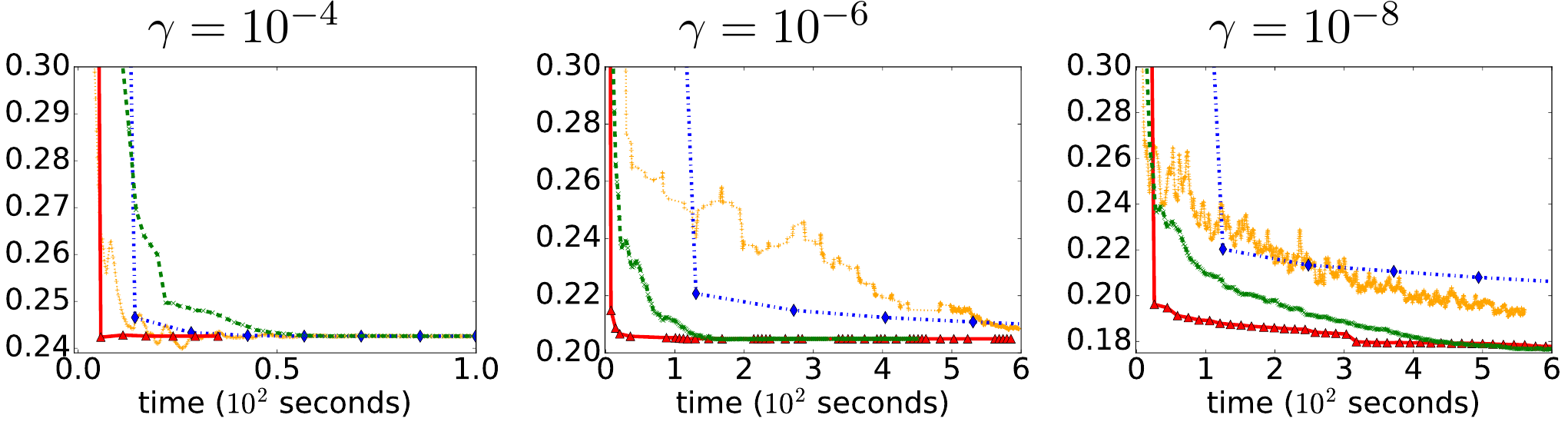}}
		\includegraphics[width=0.5\textwidth]{figure/spark/legend.pdf}
	\end{center}
	\caption{Results on the Covertype data with random feature mapping ($n=481K$ and $d=10K$).
		The feature matrix is dense.
		The regularization parameter $\gamma$ is defined in~\eqref{eq:def:logis}.
		We uses $15$ compute nodes (totally $480$ cores) and partition the data to $m=89$ parts.
	}
	\label{fig:covtype}
	\vspace{-2mm}
\end{figure*}
%---------------------------------Figure---------------------------------%

We {\bf do not} compare with CoCoA \cite{ma2015adding}, DiSCO \cite{zhang2013communication}, and AIDE \cite{reddi2016aide} for the following reasons.

\begin{itemize}[leftmargin=*,wide=0em, itemsep=-2pt,topsep=-2pt, label = {\bfseries -}]
	\item 
	\textbf{CoCoA.} The local sub-problems of CoCoA are the dual problems of logistic regression, which is indeed a \emph{constrained} problem and is, computationally, much more expensive than unconstrained optimization. Unfortunately, in \cite{ma2015adding}, we did not find an appropriate description of how to solve such constrained sub-problems \emph{efficiently}. 
	\item 
	\textbf{DiSCO.}
	In each iteration, each worker machine is merely charged with performing a matrix-vector multiplication, while the driver must solve a $d\times d$ linear system. When $d$ is small, DiSCO can be efficient. 
	When $ d $ is at the thousand scale, most computations are performed by the driver rather than the workers, which are mostly left idle. 
	%In this sense, the parallelism inherent in DiSCO is rather low.
	When $d=10^4$, solving the $d\times d$ linear system on the driver machine will make DiSCO infeasible. 
	\item 
	\textbf{AIDE} is an ``accelerated'' version of DANE.
	AIDE invokes DANE as its sub-routine and has one more tuning parameter.
	However, unlike what we had expected, in all of our off-line small-scale numerical simulations, DANE consistently outperformed AIDE (both with line search).%\footnote{We show the evidence in Appendix~\ref{sec:similations}.}
	We believe that the Nesterov acceleration does not help make Newton-type method faster.
	Hence, we opted not to spend our limited budget of Cori CPU hours to conduct a comparison with AIDE.
	%\footnote{We have informed their primary author of our observation and asked for their experiment setting, but we have not got his reply.}	Thus we do not implement AIDE using Spark and conduct expensive experiments.
\end{itemize}

\subsection{Line Search}

For the compared methods---GIANT, L-BFGS, DANE---we use the backtracking line search to determine the step size.
Specifically, let $\pp$ be a computed descending direction at $\w$, we seek to find a step size $\alpha$ that satisfies the Armijo–Goldstein condition:
\begin{eqnarray*}
	f (\w + \alpha \pp )
	\; \leq \; f (\w ) + \alpha c \big\langle \pp , \, \nabla f(\w) \big\rangle ,
\end{eqnarray*}
where $f$ is the objective function.
Throughout, we fix the control parameter to $c=0.1$ and select the step size, $\alpha$, from the candidate set $\AM = \{4^{0}, 4^{-1}, \cdots , 4^{-9}\}$.
These line search parameters are problem-independent and data-independent and do not need tuning.
According to our off-line experiments, the tuning of these parameter does not demonstrate substantial improvement to the convergence.

Line search requires two extra rounds of communications.
First, the driver \textsf{Broadcasts} $\pp \in \RB^d$ to all the worker machines, and subsequently, every worker machine (say the $i$-th) computes its local objective values $f_i (\w + \alpha \pp )$ for $\alpha$ in the set of candidate step sizes, $\AM$.
Second, the driver sums the local objective values by a \textsf{Reduce} operation and obtain $f (\w + \alpha \pp )$ for $\alpha \in \AM $.
Then the driver can locally select a step size from $\AM$ which satisfies the Armijo–Goldstein condition. 

The communication complexity of line search is $\tilde{\OM }  (d )$, which is the same as computing the gradient.
The local computational cost of line search is at most $| \AM |$ times higher than computing the gradient.

\subsection{Experiments on Time Efficiency}

We used three classification datasets: MNIST8M ($n=8M$ and $d=784$), Epsilon ($n=500K$ and $d=2K$), and Covtype ($n=581K$ and $d=54$), which are available at the LIBSVM website.
We randomly hold $80\%$ for training and the rest for test.
To increase the size of the data, we generate $10^4$ random Fourier features \cite{rahimi2007random} and use them in lieu of the original features in the logistic regression problem.
We use the RBF kernel $k (\x_i, \x_j) = \exp (- \tfrac{1}{2\sigma } \| \x_i - \x_j \|_2^2 )$ and fix $\sigma$ as
$\sigma = 
{ \sum_{i,j} \| \x_i - \x_j \|_2^2}$.

We use different settings of the regularization parameter $\gamma$, which affects the condition number of the Hessian matrix and thereby the convergence rate.
We report the results in Figures~\ref{fig:mnist}, \ref{fig:epsilon}, and~\ref{fig:covtype} which clearly demonstrate the superior performance of GIANT.
Using the same amount of wall-clock time, GIANT converges faster than AGD, DANE, and L-BFGS in terms of both training objective value and test classification error.

Our theory requires the local sample size $s = \frac{n}{m}$ to be larger than $d$.
But in practice, GIANT converges even if $s$ is smaller than $d$.
In this set of experiments, we set $m = 89$, and thus $s$ is about half of $d$.
Nevetheless, GIANT with line search converges in all of our experiments.

%---------------------------------Figure---------------------------------%
\begin{figure*}[!t]
	\begin{center}
		\subfigure[Objective function value $f(\w_t) - f(\w^\star)$ against the wall-clock time.]{\includegraphics[width=0.9\textwidth]{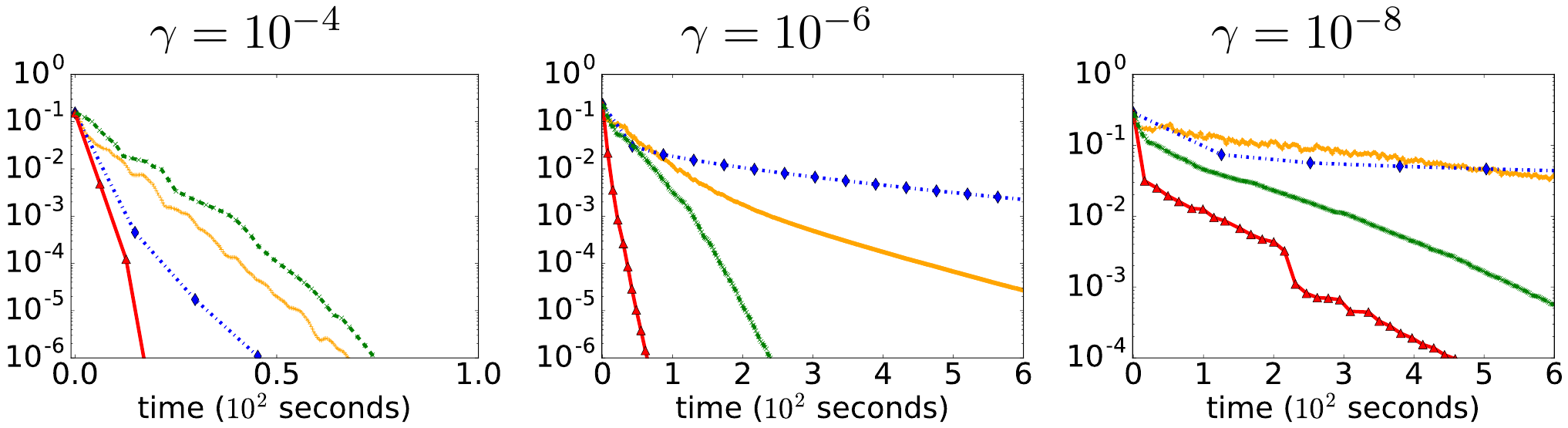}}
		\subfigure[Classification error rate (\%) on the test set against the wall-clock time.]{\includegraphics[width=0.9\textwidth]{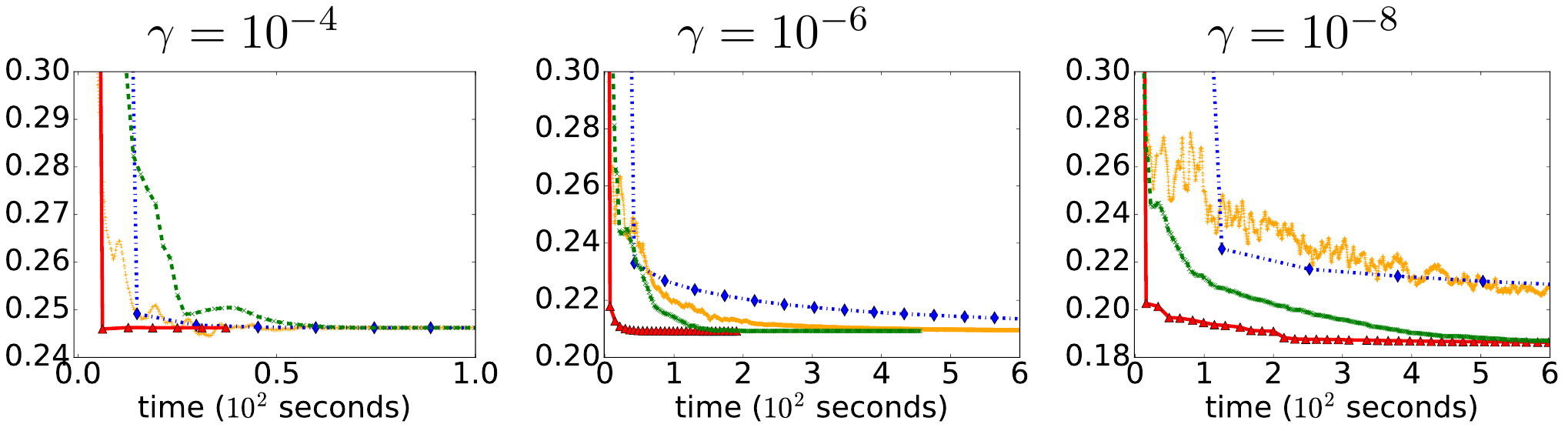}}
		\includegraphics[width=0.5\textwidth]{figure/spark/legend.pdf}
	\end{center}
	\caption{Results on the augmented Covertype data with random feature mapping ($n=2.4M$ and $d=10K$).
		The feature matrix is dense.
		The regularization parameter $\gamma$ is defined in~\eqref{eq:def:logis}.
		We uses $75$ compute nodes (totally $2400$ cores) and partition the data to $m=449$ parts.
	}
	\label{fig:covtype75}
\end{figure*}
%---------------------------------Figure---------------------------------%

%---------------------------------Figure---------------------------------%
\begin{figure*}[!t]
	\begin{center}
		\subfigure[Objective function value $f(\w_t) - f(\w^\star)$ against the wall-clock time.]{\includegraphics[width=0.9\textwidth]{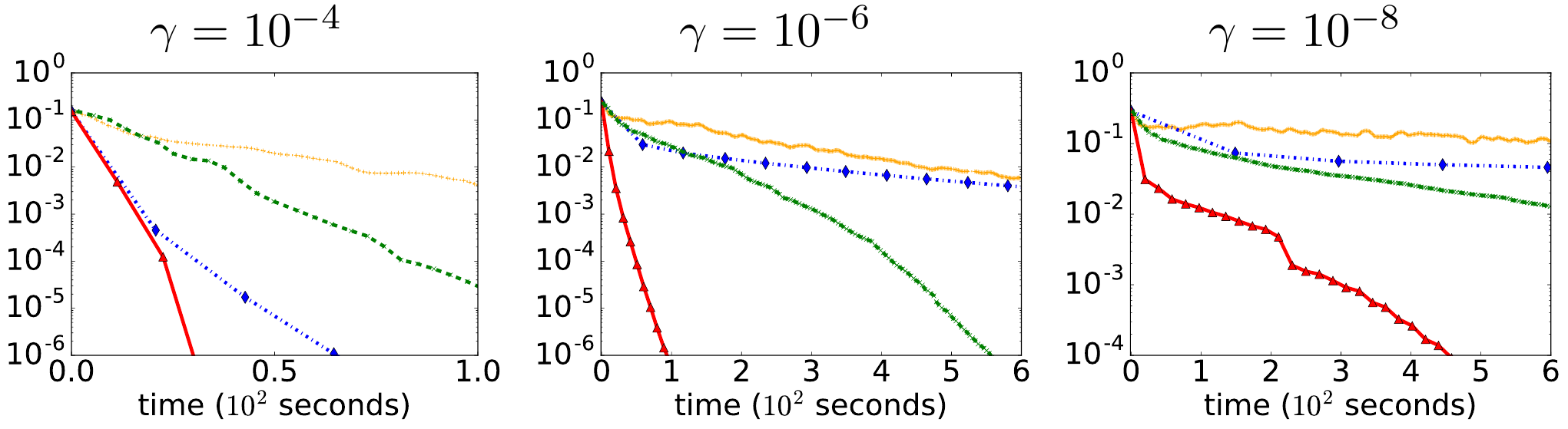}}
		\subfigure[Classification error rate (\%) on the test set against the wall-clock time.]{\includegraphics[width=0.9\textwidth]{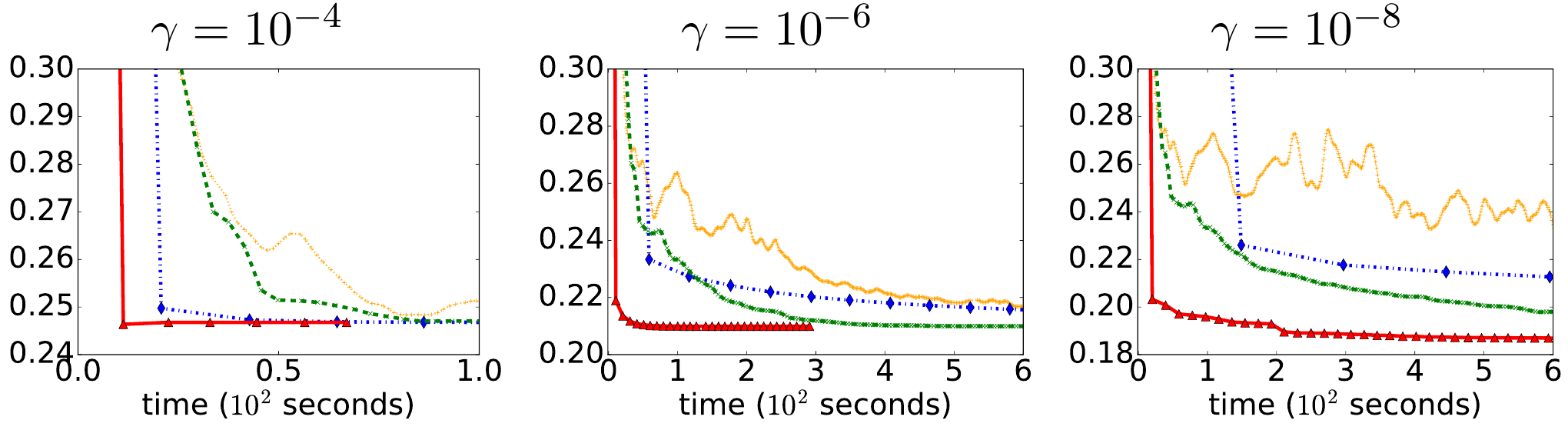}}
		\includegraphics[width=0.5\textwidth]{figure/spark/legend.pdf}
	\end{center}
	\caption{Results on the augmented Covertype data with random feature mapping ($n=12M$ and $d=10K$).
		The feature matrix is dense.
		The regularization parameter $\gamma$ is defined in~\eqref{eq:def:logis}.
		We uses $375$ compute nodes (totally $12,000$ cores) and partition the data to $m=2249$ parts.
	}
	\label{fig:covtype375}
\end{figure*}
%---------------------------------Figure---------------------------------%

\subsection{Experiments on Scalability}

To test the scalability of the compared methods, we increase the number of samples by a factor of $k$ by data augmentation.
We replicate $X$ and $y$ and stack them to form a $kn\times d$ feature matrix and a $kn$-dimensional label vector.
We inject i.i.d.\ Gaussian noise $\NM (0, 0.02^2)$ to every entry of the obtained feature matrix.
We do the $80\%$---$20\%$ random partition to get training and test sets and then the random feature mapping.
Because the data get $k$ times larger, we accordingly use $k$ times more compute nodes.
We set $k=5$ and report the results in Figure~\ref{fig:covtype75}; we set $k=25$ and report the results in Figure~\ref{fig:covtype375}.

Figures \ref{fig:covtype}, \ref{fig:covtype75}, and \ref{fig:covtype375} respectively show the results on the $n\times d$ dataset, the $5n\times d$ dataset, and the $25n\times d$ data. 
For the $k$ ($k=5$ or $25$) times larger data, we use $k$ times more compute nodes in order that the local computation per iteration remains the same.
For AGD and L-BFGS, the convergence of the objective function in Figure~\ref{fig:covtype75} is slower than in Figure~\ref{fig:covtype}, because with $5$ times more nodes, communication is slightly more expensive.
In contrast, the communication-efficient methods, GIANT and DANE, are almost unaffected.

%---------------------------------Figure---------------------------------%
\begin{figure*}[!t]
	\begin{center}
		\subfigure[AGD.]{\includegraphics[width=0.31\textwidth]{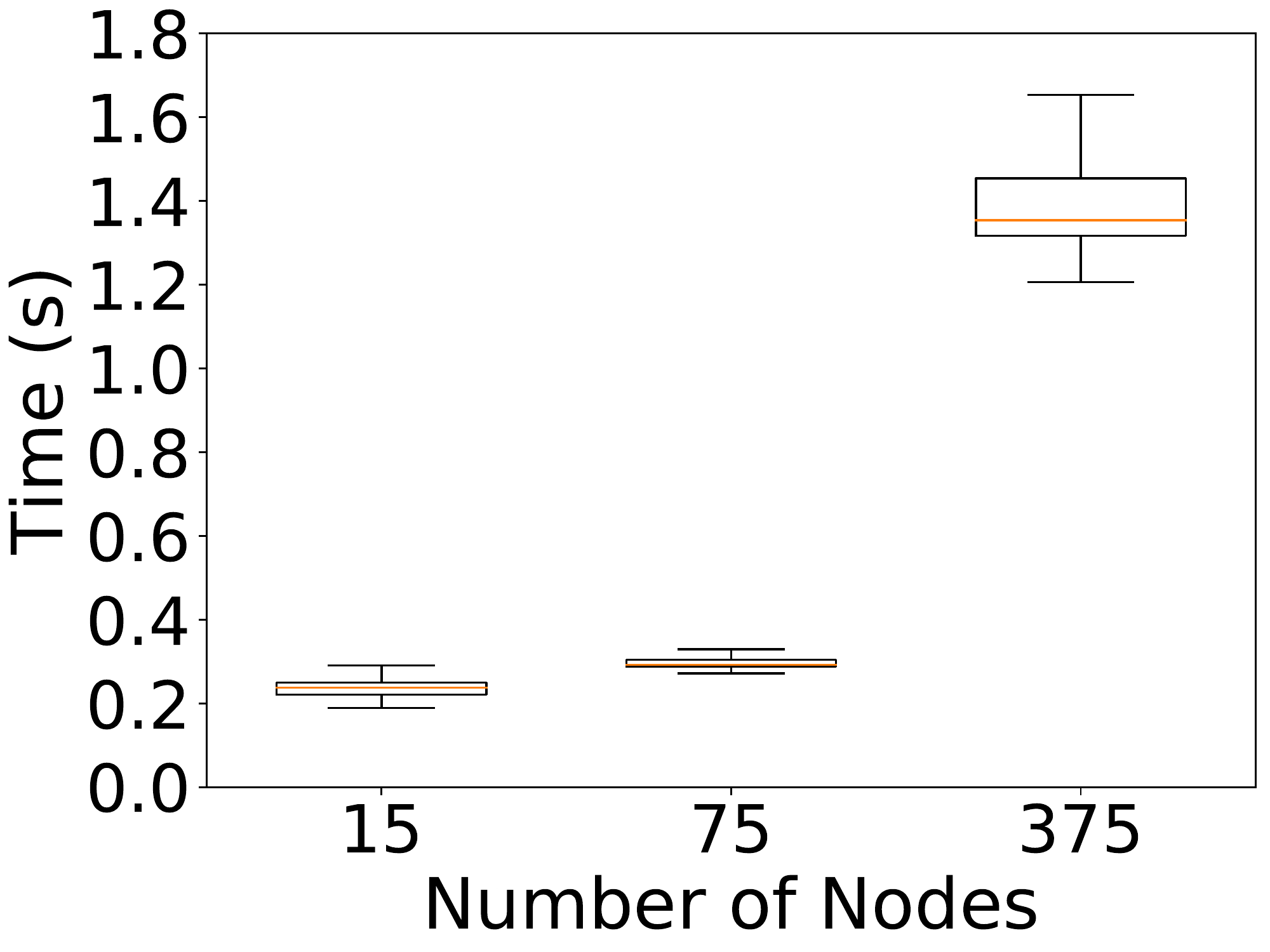}}~~
		\subfigure[L-BFGS.]{\includegraphics[width=0.31\textwidth]{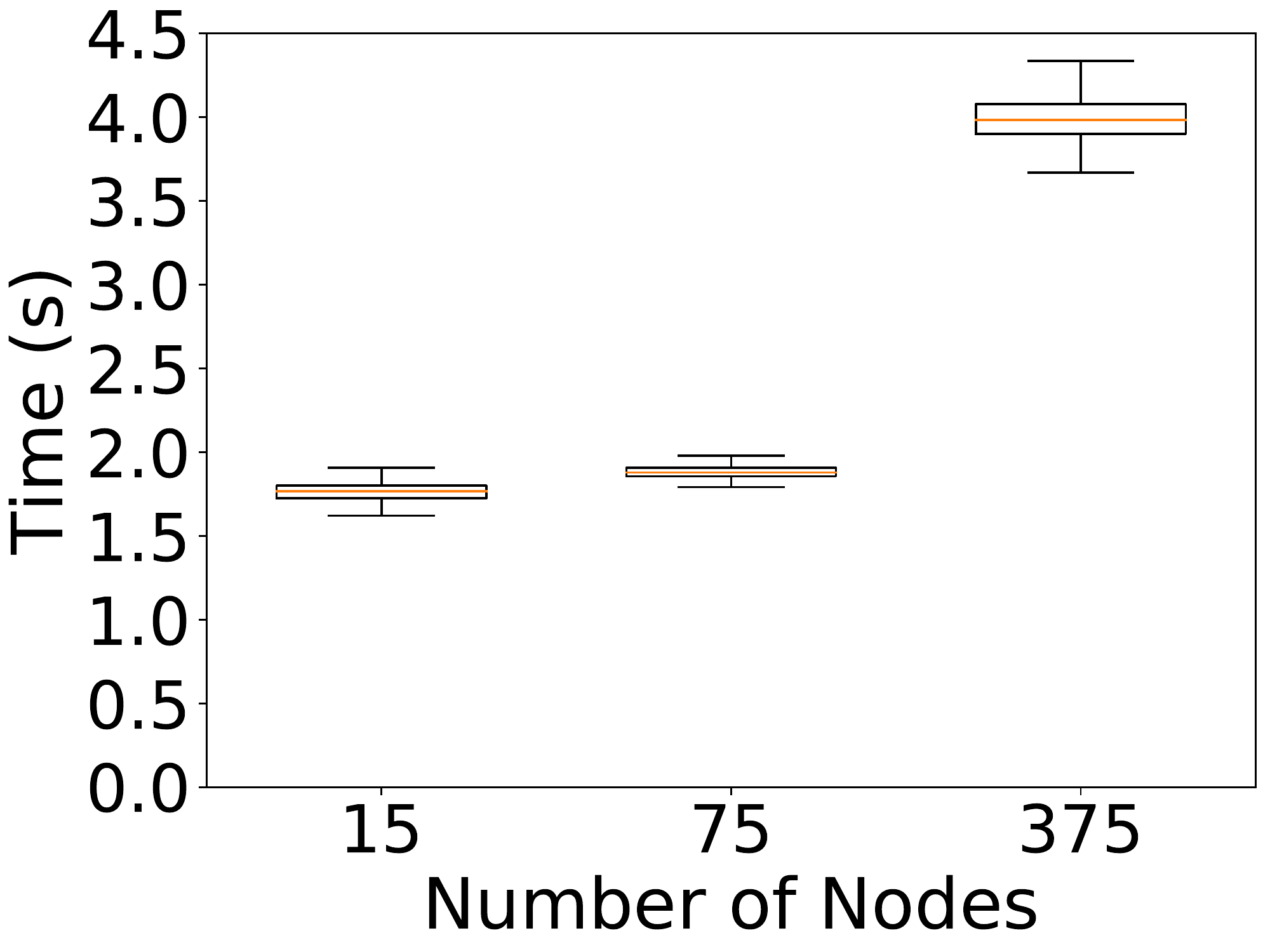}}~~
		\subfigure[GIANT.]{\includegraphics[width=0.31\textwidth]{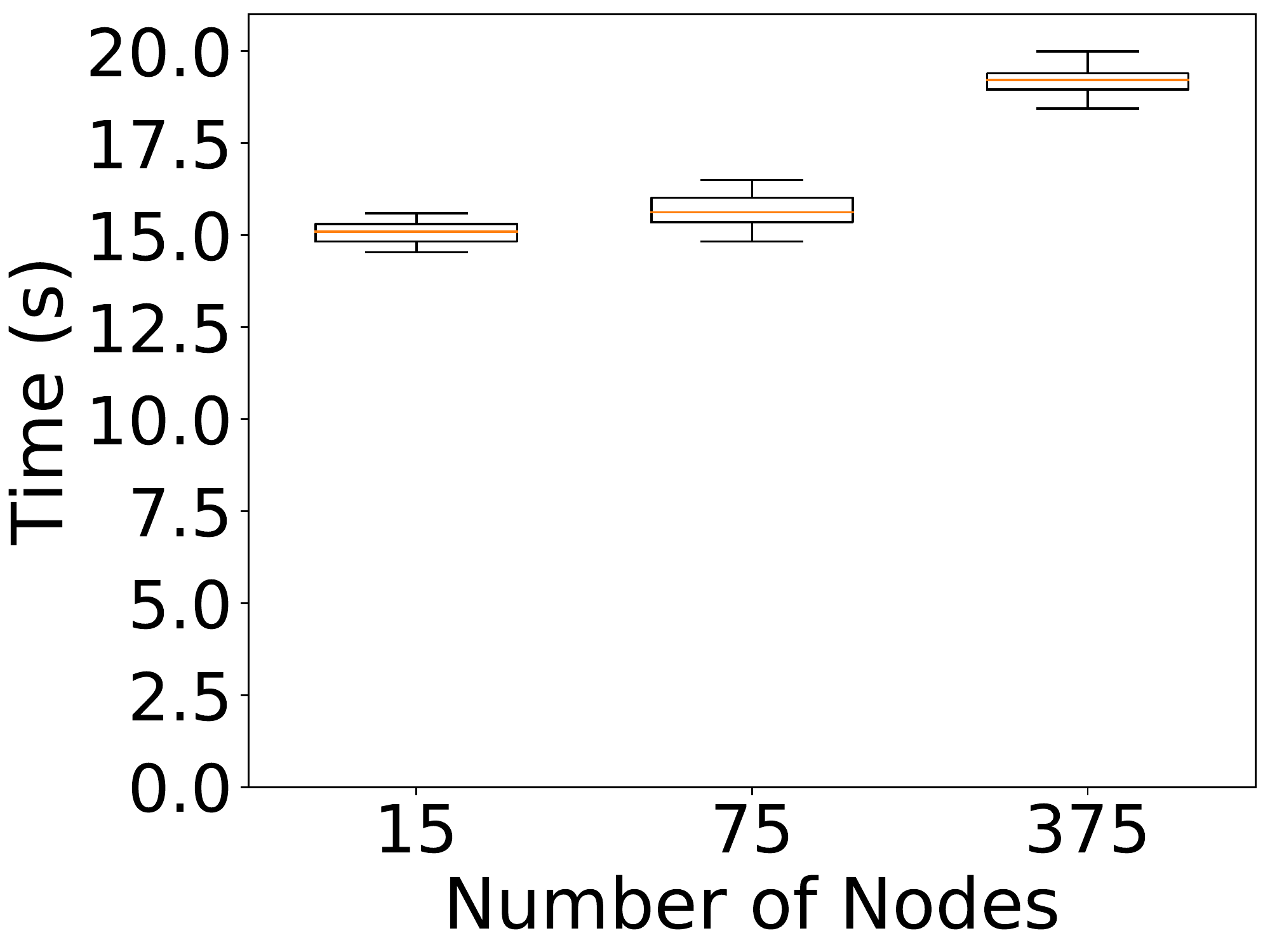}}
	\end{center}
	\caption{The per-iteration wall-clock time against the number of compute nodes.
		Here we use the Covtype dataset and set the regularization parameter $\gamma = 10^{-8}$;
		for GIANT, we set the number of CG iterations to $100$.
		As the number of nodes increases from 75 to 375, the increases of the per-iteration time are respectively $367\%$ (AGD), $112\%$ (L-BFGS), and $19\%$ (GIANT).
	}
	\label{fig:scalibility}
\end{figure*}
%---------------------------------Figure---------------------------------%

Now we explain why GIANT is more scalable than AGD and L-BFGS.
On the one hand, using more compute nodes, the \textsf{Broadcast} and \textsf{Reduce} of a vector from/to all the nodes become more expensive, and the straggler's effect (i.e., everyone waits for the slowest to complete) deteriorates.
In short, the communication and synchronization costs increase rapidly with the number of nodes.
On the other hand, since the size of data on each node does not vary, the per-iteration local computation remains the same.
AGD and L-BFGS are highly iterative: in each iteration, they do a little computation and 2 rounds of communication. 
Thus the per-iteration time costs of AGD and L-BFGS increase significantly with the number of nodes; see Figure~\ref{fig:scalibility}.
GIANT is computation intensive: in each iteration, GIANT does much computation and just 6 rounds of communication (including the line search); since the cost is dominated by the local computation, the increase in the communication cost only marginally affects the total runtime.

%%%%%%%%%%%%%%%%%%%%%%%%%%%%%%%%%%%%%%%%%%%%%%%%%%%%%%%%%%%%%%%%%%%%%%%%%%%%%%
%%%%%%%%%%%%%%%%%%%%%%%%%%%%%%%%%%%%%%%%%%%%%%%%%%%%%%%%%%%%%%%%%%%%%%%%%%%%%%
%%%%%%%%%%%%%%%%%%%%%%%%%%%%%%%%%%%%%%%%%%%%%%%%%%%%%%%%%%%%%%%%%%%%%%%%%%%%%%
%%%%%%%%%%%%%%%%%%%%%%%%%%%%%%%%%%%%%%%%%%%%%%%%%%%%%%%%%%%%%%%%%%%%%%%%%%%%%%

\section{Conclusions and Future Work}

We have proposed GIANT, a practical Newton-type method, for empirical risk minimization in distributed computing environments.
In comparison to similar methods, GIANT has three desirable advantages.
First, GIANT is guaranteed to converge to high precision in a small number of iterations,
provided that the number of training samples, $n$, is sufficiently large, relative to $dm$,
where $d$ is the number of features and $m$ is the number of partitions.
Second, GIANT is very communication efficient in that each iteration requires four or six rounds of communications, each with a complexity of merely $\tilde{\OM }  (d )$.
Third, in contrast to all other alternates, GIANT is easy to use, as it involves tuning one parameter.
Empirical studies also showed the superior performance of GIANT as compared several other methods.

GIANT has been developed only for unconstrained problems with smooth and strongly convex objective function. However, we believe that similar ideas can be naturally extended to {\it projected Newton} for constrained problems, {\it proximal Newton} for non-smooth regularization, and {\it trust-region method} for nonconvex problems. However, strong convergence bounds of the extensions appear nontrivial and will be left for future work.

%%%%%%%%%%%%%%%%%%%%%%%%%%%%%%%%%%%%%%%%%%%%%%%%%%%%%%%%%%%%%%%%%%%%%%%%%%%%%%
%%%%%%%%%%%%%%%%%%%%%%%%%%%%%%%%%%%%%%%%%%%%%%%%%%%%%%%%%%%%%%%%%%%%%%%%%%%%%%
%%%%%%%%%%%%%%%%%%%%%%%%%%%%%%%%%%%%%%%%%%%%%%%%%%%%%%%%%%%%%%%%%%%%%%%%%%%%%%
%%%%%%%%%%%%%%%%%%%%%%%%%%%%%%%%%%%%%%%%%%%%%%%%%%%%%%%%%%%%%%%%%%%%%%%%%%%%%%

\section*{Acknowledgement}
	We thank Kimon Fountoulakis, Jey Kottalam, Hao Ren, Sathiya Selvaraj, Zebang Shen, and Haishan Ye for their helpful suggestions.

%%%%%%%%%%%%%%%%%%%%%%%%%%%%%%%%%%%%%%%%%%%%%%%%%%%%%%%%%%%%%%%%%%%%%%%%%%%%%%
%%%%%%%%%%%%%%%%%%%%%%%%%%%%%%%%%%%%%%%%%%%%%%%%%%%%%%%%%%%%%%%%%%%%%%%%%%%%%%
%%%%%%%%%%%%%%%%%%%%%%%%%%%%%%%%%%%%%%%%%%%%%%%%%%%%%%%%%%%%%%%%%%%%%%%%%%%%%%
%%%%%%%%%%%%%%%%%%%%%%%%%%%%%%%%%%%%%%%%%%%%%%%%%%%%%%%%%%%%%%%%%%%%%%%%%%%%%%
\appendix

\input{text/sketch}

\input{text/proof}

\begin{small}
	\bibliography{bib/matrix,bib/optimization,bib/system,bib/misc}
	\bibliographystyle{plain}
\end{small}

\end{document}

%% file: text/newcommands.tex
\usepackage{graphicx} % more modern
\usepackage{subfigure}
\usepackage{wrapfig}
\usepackage{blindtext}
\usepackage{amsmath,amssymb,amsmath,amsthm}
\usepackage{framed}
\usepackage{color}
\usepackage[dvipsnames]{xcolor}
\usepackage{enumitem}

\makeatletter
\newcommand{\mylabel}[2]{#2\def\@currentlabel{#2}\label{#1}}
\makeatother

\usepackage{enumitem}
\setlist[enumerate,1]{leftmargin=*,wide=0em, label = {\bfseries \roman*.}}
\setlist[itemize,1]{leftmargin=*,wide=0em, itemsep=0pt,topsep=0pt}
\def\A{{\bf A}}
\def\a{{\bf a}}

\def\C{{\bf C}}
\def\c{{\bf c}}

\def\g{{\bf g}}

\def\H{{\bf H}}
\def\I{{\bf I}}

\def\M{{\bf M}}

\def\pp{{\bf p}}

\def\R{{\bf R}}

\def\S{{\bf S}}

\def\T{{\bf T}}

\def\U{{\bf U}}
\def\u{{\bf u}}
\def\V{{\bf V}}
\def\v{{\bf v}}

\def\w{{\bf w}}
\def\X{{\bf X}}
\def\x{{\bf x}}

\def\y{{\bf y}}

\def\0{{\bf 0}}
\def\1{{\bf 1}}

\def\AM{{\mathcal A}}

\def\EM{{\mathcal E}}

\def\JM{{\mathcal J}}

\def\NM{{\mathcal N}}
\def\OM{{\mathcal O}}

\def\RB{{\mathbb R}}
\def\EB{{\mathbb E}}
\def\PB{{\mathbb P}}

\def\Si{\mbox{\boldmath$\Sigma$\unboldmath}}

\def\Gam{\mbox{\boldmath$\Gamma$\unboldmath}}

\def\De{\mbox{\boldmath$\Delta$\unboldmath}}

\def\Ups{\mbox{\boldmath$\Upsilon$\unboldmath}}

\def\argmin{\mathop{\rm argmin}}

\def\nnz{\mathrm{nnz}}

\def\sgn{\mathrm{sgn}}

\def\rk{\mathrm{rank}}

\def\etal{{\em et al.\/}\,}

\newtheorem{theorem}{Theorem}
\newtheorem{lemma}[theorem]{Lemma}
\newtheorem{corollary}[theorem]{Corollary}
\newtheorem{proposition}[theorem]{Proposition}
\newtheorem{remark}{Remark}
\newtheorem{definition}{Definition}

\newtheorem{assumption}{Assumption}

%% file: text/sketch.tex
\section{Proof of the Main Results} \label{sec:proof}

In this section we prove the main results: Theorems~\ref{thm:quadratic}, \ref{thm:general}, \ref{thm:inexact}
and Corollary~\ref{cor:general}.
In Section~\ref{sec:proof:notation} we first define the notation used throughout.
In Sections~\ref{sec:proof:avg} to \ref{sec:proof:newton} we establish or cite several lemmas
which are applied to prove our main theorems.
Lemmas~\ref{lem:avg}, \ref{lem:avg2}, and \ref{lem:convergence} are new and may have independent interest beyond the scope of this paper.
In Section~\ref{sec:proof:theorem} we complete the proof of the main theorems and corollaries.

Here we consider a problem more general than \eqref{eq:problem}:
\begin{small}
	\begin{eqnarray} \label{eq:problem2}
	\min_{\w \in \RB^d } \;
	\bigg\{  f (\w )
	\; \triangleq \; 
	\frac{1}{n} \sum_{j=1}^n  \ell_j (\w^T \x_j) + \tfrac{1}{2} \w^T \M \w \bigg\} .
	\end{eqnarray}
\end{small}%
Here $\M$ is symmetric positive semi-definite and can be set to all-zero matrix (only if the loss function is strongly convex),
$\gamma \I_d$ for some $\gamma \geq 0$, 
a diagonal matrix with non-negative entries,
or some sparse matrix which can be efficiently transferred across the network by message passing.

\subsection{Notation} \label{sec:proof:notation}

\paragraph{Matrix sketching.}
Here, we briefly review matrix sketching methods that are commonly used for randomized linear algebra (RLA) applications~\cite{mahoney2011ramdomized}.
Given a matrix $\A \in \RB^{n \times d}$, 
we refer to $\C = \S^T \A \in \RB^{s \times d}$ as \emph{sketch} of $\A$ 
with the \emph{sketching matrix} $\S \in \RB^{n\times s}$ (typically $s \ll n$). In many RLA applications, the rows of $\C$ are typically  made up of a randomly selected and rescaled subset of the rows of $\A$, or their random linear combinations;
the former type of sketching is called \emph{row selection} or \emph{random sampling}, and the latter is referred to as \emph{random projection}.
Such randomized sketching has emerged as a powerful primitive in RLA for dealing with large-scale matrix computation problems~\cite{mahoney2011ramdomized,drineas2016randnla}.
This is mainly due to the fact that sketching, if done right, allows for large matrices to be ``represented'' by smaller alternatives which are more amenable to efficient computations and storage, while provably retaining certain desired properties of the original matrices~\cite{mahoney2011ramdomized,woodruff2014sketching}.

We consider matrix multiplication formulation of \emph{row selection} in which the sketched matrix, $\C \in \RB^{s \times d}$, is constructed using a randomly sampled and particularly rescaled subset of the rows of $\A \in \RB^{n \times d}$.
More specifically, let $p_1 , \cdots , p_n \in (0, 1)$ be the sampling probabilities associated with the rows of $\A$ (so that, in particular, $\sum_{i=1}^n p_i = 1$). 
The rows of the sketch are selected independently and according to the sampling distribution $ \{p_{i}\}_{i=1}^{n} $ such that we have
\begin{align*}
\PB \big(\c_i = \a_j / \sqrt{s p_j} \big) = p_j, \quad
\textrm{ for all } \; j = 1,2,\ldots,n,
\end{align*}
where $ \c_i $ and $ \a_j $ are $ i^{th} $ and $ j^{th} $ rows of $ \C $ and $ \A $, respectively.
As a result, the sketching matrix $\S \in \RB^{n \times s}$
contains exactly one non-zero entry in each column, 
whose position and magnitude correspond to the selected row of $\A$. 
\emph{Uniform sampling} is a particular form of row sampling with $p_1 = \cdots = p_n = 1/n$, while \emph{leverage score sampling} takes $p_i$ proportional to the $i$-th leverage score of $\A$ for $i\in[n]$ 
(or its randomized approximation~\cite{drineas2012fast}).

\emph{Random projection} forms a sketch by taking random linear combinations of the rows of $\A$.
Popular random projections include, among many others, 
Gaussian projection \cite{johnson1984extensions},
subsampled randomized Hadamard transform~\cite{drineas2011faster,lu2013faster,tropp2011improved},
Rademacher random variables~\cite{achlioptas2003database},
CountSketch~\cite{clarkson2013low,meng2013low,nelson2013osnap}.

\paragraph{The quadratic function $\phi$.}
Let $\A_t$ be defined in \eqref{eq:def:at}.
%Let $\c_t = \bb_t + (\A_t^T)^\dag \M \w_t$.
We define the auxiliary function
\begin{eqnarray} \label{eq:def:phi}
\phi_t (\pp) \: \triangleq \:
\tfrac{1}{2} \pp^T \underbrace{(\A_t^T \A_t + \M )}_{\triangleq \H_t} \pp -  \pp^T \g_t
\end{eqnarray}
and study its approximate solutions in this section.
The Newton direction at $\w_t$ can be written as
\begin{eqnarray*} 
	\pp_t^\star
	\; = \;
	\argmin_{\pp}
	\, \phi_t (\pp) 
	\; = \; (\A_t^T \A_t + \M )^{-1} \g_t .
\end{eqnarray*}
Let $\S_1 , \cdots , \S_m$ be some sketching matrices.
By definition, the local and global approximate Newton directions are respectively
\begin{eqnarray} \label{eq:def:pt}
\tilde{\pp }_{t,i} 
\; = \;(\A_t^T \S_i \S_i^T \A_t + \M )^{-1} \g_t
\qquad \textrm{and} \qquad
\tilde{\pp}_t
\; = \; \frac{1}{m} \sum_{i=1}^m \tilde{\pp }_{t,i} .
\end{eqnarray}
It can be verified that
$\tilde{\pp}_{t,i}$ is the minimizer of the sketched problem
\begin{eqnarray} \label{eq:def:tilde_phi}
\tilde{\phi}_{t,i} (\pp) 
\: \triangleq \:
\tfrac{1}{2} \pp^T \underbrace{(\A_t^T \S_i \S_i^T \A_t + \M )}_{\triangleq \widetilde{\H}_{t,i}} \pp
-  \pp^T \g_t  .
\end{eqnarray}
We will show that $\tilde{\pp}_t$ is close to $\pp_t^\star$ in terms of the value of the function $\phi_t (\cdot )$.
This is the key to the convergence analysis of GIANT.

\paragraph{Singular value decomposition (SVD).}
Let $\A \in \RB^{n\times d}$ and $\rho = \rk (\A)$.
A (compact) singular value decomposition (SVD) is defined by 
\begin{eqnarray*}
	\textstyle{\A \; = \; \U \Si \V^T
		\; = \; \sum_{i=1}^\rho  \sigma_i \u_i \v_i^T ,}
\end{eqnarray*}
where 
$\U$, $\Si$, $\V$ are a $n\times \rho$ column-orthogonal matrix,
a $\rho\times \rho$ diagonal matrix with nonnegative entries, and a $d\times \rho$ column-orthogonal matrix, respectively.
If $\A$ is symmetric positive semi-definite (SPSD), then $\U = \V$, and this decomposition is, at times, referred to as the (reduced) eigenvalue decomposition (EVD). By convention,
singular values are ordered such that $\sigma_1 \geq \cdots \geq \sigma_\rho$.

\subsection{Analysis of Model Averaging}\label{sec:proof:avg}

Lemma~\ref{lem:avg} shows that $\phi_t(\tilde{\pp}_t)$ is close to $\phi_t(\pp_t^\star)$.
Note that $\phi_t(\tilde{\pp}_t)$ and $\phi_t(\pp_t^\star)$ are both non-positive.
The proof of Lemma~\ref{lem:avg} uses some techniques developed by \cite{wang2017sketched}.
We prove Lemma~\ref{lem:avg} in Appendix~\ref{proof:lem:avg}.

\begin{assumption}\label{assumption:avg}
	Let $\eta \in (0, 1)$ be any fixed parameter,
	$\rho = \rk (\A_t)$,
	and $\U \in \RB^{n\times \rho}$ be the orthogonal bases of $\A_t$.
	Let $\S_1 , \cdots , \S_m \in \RB^{n\times s}$ be certain sketching matrices
	and $\S = \frac{1}{\sqrt{m}} [\S_1 , \cdots , \S_m ] \in \RB^{n\times ms}$;
	here $s$ depends on $\eta$.
	Assume $\S_i$ and $\S$ satisfy
	\begin{eqnarray*}
		\big\|\U^T \S_i \S_i^T \U - \I_\rho \big\|_2 \leq \eta
		\quad \textrm{for all} \; i \in [m] 
		\qquad \textrm{and} \qquad 
		\big\|\U^T \S \S^T \U - \I_\rho \big\|_2 \leq \tfrac{\eta}{ \sqrt{m}} .
	\end{eqnarray*}
\end{assumption}

\begin{lemma} [Exact Solution to Subproblems] \label{lem:avg}
	Let $\S_1 , \cdots , \S_m \in \RB^{n\times s}$ satisfy Assumption~\ref{assumption:avg}.
	Let $\phi_t$ be defined in \eqref{eq:def:phi}
	and $\tilde{\pp}_{t}$ be defined in \eqref{eq:def:pt}.
	It holds that
	\begin{eqnarray*}
		\min_\pp \phi_t (\pp )
		\; \leq \;
		\phi_t (\tilde{\pp}_{t} )
		\; \leq \; \big( 1 - \alpha^2 \big)
		\cdot \min_\pp \phi_t (\pp ) ,
	\end{eqnarray*}
	where $\alpha = \vartheta \big(\tfrac{\eta}{\sqrt{m}} + \tfrac{\eta^2 }{1- \eta } \big)$
	and $\vartheta = \tfrac{ \sigma_{\max} (\A_t^T \A_t ) }{ \sigma_{\max} (\A_t^T \A_t ) + \sigma_{\min} (\M) } \leq 1$.
\end{lemma}

Lemma~\ref{lem:avg} requires the exact minimizer to $\tilde{\phi}_{t,i} (\cdot )$ in \eqref{eq:def:tilde_phi},
denote $\tilde{\pp}_{t,i}$, which requires the computation of 
$\widetilde{\H}_{t, i} = \A_t^T \S_i \S_i^T \A_t + \M$ and its inversion.
Alternatively, we can use numerical optimization, such as CG,
to optimize $\tilde{\phi}_{t,i} (\cdot )$ up to a fixed precision.
We denote the inexact solution as $\tilde{\pp}_{t,i}'$ 
and assume it is close to the exact solution $\tilde{\pp}_{t,i}$.
Then $\tilde{\pp}_{t,i}'$ is also a good approximation to $\pp_t^\star$
in terms of the values of $\phi_t (\cdot )$.
We prove Lemma~\ref{lem:avg2} in Appendix~\ref{proof:lem:avg2}.

\begin{lemma} [Inexact Solution to Subproblems] \label{lem:avg2}
	Let $\S_1 , \cdots , \S_m \in \RB^{n\times s}$ satisfy Assumption~\ref{assumption:avg}.
	Let $\tilde{\pp}_{t,i} $ and $\widetilde{\H}_{t,i}$ be defined in \eqref{eq:def:pt} and \eqref{eq:def:tilde_phi}, respectively.
	Let $\tilde{\pp}_{t,i}'$ be any vector satisfying 
	\[
	\Big\|  \widetilde{\H}_{t,i}^{1/2} 
	\big( \tilde{\pp}_{t,i}' - \tilde{\pp}_{t,i} \big) \Big\|_2
	\; \leq \; \epsilon_0 \Big\| \widetilde{\H}_{t,i}^{1/2}  \tilde{\pp}_{t,i}  \Big\|_2
	\]
	for some fixed $\epsilon_0 \in (0, 1)$.
	Let $\tilde{\pp}_{t}' = \tfrac{1}{m} \sum_{i=1}^m \tilde{\pp}_{t,i}' $.
	Let $\phi_t$ be defined in \eqref{eq:def:phi}.
	It holds that
	\begin{eqnarray*}
		\min_\pp \phi_t (\pp )
		\; \leq \;
		\phi_t (\tilde{\pp}' )
		\; \leq \; 
		\big(1 - \alpha'^2 \big)
		\, \cdot \, \min_\pp \phi_t (\pp ) ,
	\end{eqnarray*}
	where $\alpha' = \vartheta \big(\tfrac{\eta}{\sqrt{m}} + \tfrac{\eta^2 }{1- \eta } \big)
	+\tfrac{ \epsilon_0  }{1-\eta  } $
	and $\vartheta = \tfrac{ \sigma_{\max}^2 (\A_t ) }{ \sigma_{\max}^2 (\A_t ) + \sigma_{\min} (\M) } \leq 1$.
\end{lemma}

\subsection{Analysis of Uniform Sampling}\label{sec:proof:uniform}

Lemma~\ref{lem:avg} and \ref{lem:avg2} require the sketching matrices satisfying some properties.
Lemma~\ref{lem:uniform}, which is cited from \cite{wang2017sketched},
shows that uniform sampling matrices enjoys the properties in Assumption~\ref{assumption:avg}
for an appropriately-chosen sample size $s$.
The proof of Lemma~\ref{lem:uniform} is based on the results in \cite{drineas2006sampling,drineas2011faster,woodruff2014sketching,wang2016spsd}.

\begin{lemma} 
	\label{lem:uniform}
	Let $\eta , \delta \in (0, 1)$ be fixed paramters.
	Let $\U \in \RB^{n\times \rho }$ be any fixed matrix with orthonormal columns.
	Let $\S_1, \cdots , \S_m \in \RB^{n\times s}$ be independent uniform sampling matrices
	with $s \geq  \frac{3 \mu \rho }{\eta^{2} } \log \frac{\rho m }{ \delta}  $
	and $\S \in \RB^{n\times ms}$ be the concatenation of $\S_1 , \cdots , \S_m$.
	It holds with probability at least $1-\delta$ that
	\begin{eqnarray*}
		\big\|\U^T \S_i \S_i^T \U - \I_\rho \big\|_2 \leq \eta
		\quad \textrm{for all} \; i \in [m] 
		\qquad \textrm{and} \qquad 
		\big\|\U^T \S \S^T \U - \I_\rho \big\|_2 \leq \tfrac{\eta}{\sqrt{m}} .
	\end{eqnarray*}
\end{lemma}

Besides uniform sampling, 
other sketching methods such as leverage score sampling \cite{drineas2006sampling},
Gaussian projection \cite{johnson1984extensions},
Rademacher random variables~\cite{achlioptas2003database},
subsampled randomized Hadamard transform \cite{tropp2011improved,drineas2011faster,lu2013faster}
also satisfy Assumption~\ref{assumption:avg}.
In addition, these sketching methods eliminate the dependence of $s$ on the matrix coherence $\mu$.
These sketches are more expensive to implement than simple uniform sampling (e.g., to compute approximations to the leverage scores with the algorithm of~\cite{drineas2012fast} takes roughly the amount of time it takes to implement a random projection), and an obvious question raised by our results is whether there exists a point in communication-computation tradeoff space where using these sketches would be better than performing simple uniform sampling.

\subsection{Analysis of the Approximate Newton Step}\label{sec:proof:newton}

Let $\phi_t$ be defined in \eqref{eq:def:phi};
note that $\phi_t$ is non-positive.
If the approximate Newton direction $\tilde{\pp}_t$ is close to the exact Newton step $\pp^\star_t$
in terms of $\phi_t$,
then $\tilde{\pp}$ is provably a good descending direction.
The proof follows the classical local convergence analysis of Newton's method \cite{wright1999numerical}.
We prove Lemma~\ref{lem:convergence} in Appendix~\ref{sec:proof:convergence}.

\begin{lemma} \label{lem:convergence}
	Let Assumption~\ref{assumption:lipschitz} (the Hessian matrix is $L$-Lipschitz) hold.
	Let $\alpha \in (0, 1)$ be any fixed error parameter.
	Assume $\tilde{\pp}_t$ satisfy  
	\begin{eqnarray*}
		\phi_t (\tilde{\pp}_{t} )
		& \leq & \big( 1 - \alpha^2 \big)
		\cdot \min_\pp \phi_t (\pp ) .
	\end{eqnarray*}
	Then $\De_t = \w_t - \w^\star$ satisfies
	\begin{eqnarray*}
		\De_{t+1}^T \H_t \De_{t+1} 
		& \leq &
		L\big\| \De_t \big\|_2^2  \big\|\De_{t+1} \big\|_2
		+  \tfrac{\alpha^2 }{1 - \alpha^2 } \, \De_t^T \H_t \De_t  .
	\end{eqnarray*}
\end{lemma}

\subsection{Completing the Proofs}\label{sec:proof:theorem}

Finally, we prove our main results using the lemmas in this section.
Let $\mu_t$ be the row coherence of $\A_t$
and $s \geq  \frac{3 \mu_t d}{\eta^{2} }  \log \frac{d m }{ \delta}  $.
Let $\alpha = \vartheta (\tfrac{\eta}{\sqrt{m}} + \tfrac{\eta^2 }{1- \eta }) $,
$L$ be defined in Lemma~\ref{lem:convergence},
and $\De_t = \w_t - \w^\star$.
It follows from Lemma~\ref{lem:uniform} that uniform sampling matrices satisfy Assumption \ref{assumption:avg}
with probability (w.p.) at least $1-\delta$.
It follows from Lemmas~\ref {lem:avg} and~\ref {lem:convergence} that
\begin{eqnarray} \label{eq:proof:0}
\De_{t+1}^T \H_t \De_{t+1} 
& \leq &
L\big\| \De_t \big\|_2^2  \big\|\De_{t+1} \big\|_2
+  \tfrac{\alpha^2 }{1 - \alpha^2 } \, \De_t^T \H_t \De_t  
\end{eqnarray}
holds w.p.\ $1-\delta$.

\begin{theorem} \label{thm:quadratic2}
	Let $\mu$ be the row coherence of $\X \in \RB^{n\times d}$ and $m$ be the number of partitions.
	Assume the local sample size satisfies $s \geq \frac{3 \mu  d }{ {\textcolor{OliveGreen}{\eta}}^{2} } \log \frac{m d}{\delta} $ for some ${\textcolor{OliveGreen}{\eta}} , \delta \in (0, 1)$.
	It holds with probability $1-\delta$ that
	\begin{align*}
	\| \De_t \|_2 \leq \alpha^t \, \sqrt{ \kappa   } \, \| \De_0 \|_2,
	\end{align*}
	where 
	$\alpha = \vartheta 
	\big( \frac{\textcolor{OliveGreen}{\eta}}{\sqrt{ m } }  + {\textcolor{OliveGreen}{\eta}}^2 \big)$,
	$\vartheta = \frac{\sigma_{\max}^2 (\X )}{\sigma_{\max}^2 (\X ) + \sigma_{\min} (\M)} \leq 1$,
	and $\kappa$ is the condition number of $\nabla^2 f(\w) = \frac{1}{n} \X^T \X + \M$.
\end{theorem}

\begin{proof}
If the loss function is quadratic, then $\H (\w_0) = \H (\w_1) =  \cdots =\H (\w^\star)$;
obviously, $\H (\w)$ is $0$-Lipschitz.
Thus we let $\H \triangleq \H (\w)$ for all $\w$.
It follows from \eqref{eq:proof:0} that w.p.\ $1-\delta$,
\begin{eqnarray*}
	\De_{t+1}^T \H \De_{t+1} 
	\; \leq \; 
	\tfrac{\alpha^2 }{1 - \alpha^2 } \De_t^T \H \De_t  
	\; \leq \; \big(  \tfrac{\alpha^2 }{1 - \alpha^2 } \big)^{t+1} 
	\De_0^T \H \De_0   .
\end{eqnarray*}
It follows that w.p.\ $1-\delta$,
\begin{eqnarray}
\tfrac{ \| \De_t \|_2 }{ \| \De_0 \|_2 }
\; \leq \; \big( \tfrac{ \alpha }{ \sqrt{ 1 - \alpha^2 } } \big)^t \,
\sqrt{ \tfrac{ \sigma_{\max} (\H) }{\sigma_{\min} (\H)  }  } .
\end{eqnarray}
\end{proof}

\begin{theorem} \label{thm:general2}
	Let $\mu_t \in [1, {n}/{d}]$ be the coherence of $\A_t $ and $m$ be the number of partitions.
	Assume the local sample size satisfies $s_t \geq \frac{3 \mu_t  d }{ {\textcolor{OliveGreen}{\eta}}^{2} } \log \frac{m d}{\delta} $ for some ${\textcolor{OliveGreen}{\eta}} , \delta \in (0, 1)$.
	Under Assumption~\ref{assumption:lipschitz},
	it holds with probability $1-\delta$ that
	\begin{align*}
	\big\| \De_{t+1} \big\|_2 
	\; \leq \; \max \Big\{ \alpha \,
	\sqrt{ \tfrac{\sigma_{\max } (\H_t )}{\sigma_{\min} (\H_t )} }
	\big\| \De_t \big\|_2,  \tfrac{2L}{\sigma_{\min} (\H_t )} \big\| \De_t \big\|_2^2 \Big\} ,
	\end{align*}
	where $\alpha = \vartheta 
	\big(  \frac{\textcolor{OliveGreen}{\eta}}{\sqrt{ m }} + {\textcolor{OliveGreen}{\eta}}^2 \big)$
	and $\vartheta = \frac{\sigma_{\max}^2 (\A_t )}{\sigma_{\max}^2 (\A_t ) + \sigma_{\min} (\M)} \leq 1$.
\end{theorem}

\begin{proof}
It follows from \eqref{eq:proof:0} that w.p.\ $1-\delta$,
\begin{eqnarray*}
	\De_{t+1}^T \H_t \De_{t+1} 
	& \leq &
	\max \Big\{ 2 L \big\| \De_t \big\|_2^2  \big\|\De_{t+1} \big\|_2 , \;
	\tfrac{2 \alpha^2 }{1 - \alpha^2 }  \De_t^T \H_t \De_t  
	\Big\} .
\end{eqnarray*}
It follows that w.p.\ $1-\delta$, at least one of the following two inequalities hold:
\begin{eqnarray}\label{eq:proof:2}
\begin{split}
\big\| \De_{t+1} \big\|_2 
& \; \leq \; \tfrac{2 L}{\sigma_{\min} (\H_t )} \big\| \De_t \big\|_2^2 ,  \\
\big\| \De_{t+1} \big\|_2 
& \; \leq \; \tfrac{\alpha }{ \sqrt{1 - \alpha^2 }} \,
\sqrt{ \tfrac{2 \sigma_{\max } (\H_t )}{\sigma_{\min} (\H_t )} }
\big\| \De_t \big\|_2 ,
\end{split}
\end{eqnarray}
which proves Theorem~\ref{thm:general}.
\end{proof}

\paragraph{Proof of Corollary~\ref{cor:general}.}
Under Assumptions~\ref{assumption:lipschitz}, it holds that
$\| \H_t - \H^\star \|_2 \leq L \| \De_t \|_2$.
It follows that
\begin{eqnarray*}
	\sigma_{\min} \big( \H^\star \big) - L \| \De_t \|_2
	\; \leq \; 
	\sigma_{\min} \big( \H_t \big)
	\; \leq \; 
	\sigma_{\max} \big( \H_t \big)
	\; \leq \;
	\sigma_{\max} \big( \H^\star \big) + L \| \De_t \|_2 .
\end{eqnarray*}
It follows from Assumption~\ref{assumption:close} that
\begin{align*}
&\tfrac{\sigma_{\max } (\H_t )}{\sigma_{\min} (\H_t )}
\;\leq \;
\tfrac{\sigma_{\max } (\H^\star ) + L \| \mathbf{\Delta}_t \|_2 }{\sigma_{\min} (\H^\star ) - L \| \mathbf{\Delta}_t \|_2  }
\; \leq \; 
2 \, \tfrac{\sigma_{\max } (\H^\star )}{\sigma_{\min} (\H^\star )} ,\\
&\tfrac{L }{\sigma_{\min} (\H_t )}
\; \leq \;\tfrac{L }{\sigma_{\min} (\H^\star ) - L \| \mathbf{\Delta}_t \|_2 }
\; \leq \; \tfrac{3}{2} \tfrac{ L }{ \sigma_{\min} (\H^\star ) } .
\end{align*}
It follows from \eqref{eq:proof:2} that
\begin{small} 
	\begin{eqnarray*}
		\big\| \De_{t+1} \big\|_2 
		& \leq & \max \bigg\{ \alpha \,
		\sqrt{ \tfrac{2 \sigma_{\max } (\H_t )}{\sigma_{\min} (\H_t )} }
		\big\| \De_t \big\|_2 , \;
		\tfrac{2L}{\sigma_{\min} (\H_t )} \big\| \mathbf{\Delta}_t \big\|_2^2 \bigg\} \\
		& \leq & \max \bigg\{ 2 \alpha \,
		\sqrt{ \tfrac{\sigma_{\max } (\H^\star )}{\sigma_{\min} (\H^\star )} }
		\big\| \De_t \big\|_2 , \; 
		\tfrac{3 L}{\sigma_{\min} (\H^\star )} \big\| \mathbf{\Delta}_t \big\|_2^2 \bigg\}
	\end{eqnarray*}
\end{small}%
holds with probability at least $1-\delta$.

\paragraph{Proof of Theorem~\ref{thm:inexact}.}
In the proof of \eqref{eq:proof:0}, we replace Lemma~\ref{lem:avg} by Lemma~\ref{lem:avg2}
and $\alpha$ by $\alpha'$.
Then very similar results can be proved in the same way as
Theorems~\ref{thm:quadratic} and \ref{thm:general}
and Corollary~\ref{cor:general}.

\paragraph{Proof of Proposition~\ref{cor:cg_q}.}
Let $\pp_0$ be an arbitrary initialization.
Standard convergence bound of CG \cite{golub2012matrix} guarantees that
\begin{small}
	\begin{eqnarray*}
		\frac{ \big\| \widetilde{\H}_{t,i}^{1/2} ( \tilde{\pp}_{t,i}' - \tilde{\pp}_{t,i} ) \big\|_2^2 }{
			\big\|  \widetilde{\H}_{t,i}^{1/2} ( \pp_0 - \tilde{\pp}_{t,i} )  \big\|_2^2 }
		& \leq & 2 \Big( \frac{ \sqrt{{\tilde\kappa}_t} - 1 }{ \sqrt{{\tilde\kappa}_t } + 1 } \Big)^q ,
	\end{eqnarray*}
\end{small}%
where ${\tilde\kappa}_t $ is the condition number of $\widetilde{\H}_{t,i}$.
Let the righthand-side equal to $\frac{\epsilon_0^2}{4}$. It follows that
$q =
\log \tfrac{ 8}{\epsilon_0^2} 
\, \big/ \,
\log \tfrac{ \sqrt{{\tilde\kappa}_t } + 1 }{\sqrt{{\tilde\kappa}_t } - 1} $.
Then \eqref{eq:inexact_assumption} follows by letting $\pp_0 = \0$.
Because $(1-\eta ) \H_t \preceq \widetilde{\H}_{t,i} \preceq (1+\eta ) \H_t$, 
their condition numbers satisfy 
\[
\tfrac{1-\eta }{1 + \eta} \kappa (\H_t)
\; \leq \; {\tilde\kappa}_t 
\; \leq \; \tfrac{1+\eta }{1- \eta} \kappa (\H_t) .
\]
Clearly, ${\tilde\kappa}_t$ is very close to $\kappa_t \triangleq \kappa (\H_t)$.

%% file: text/proof.tex
\section{Proof of Lemma~\ref{lem:avg} (Model Averaging)} \label{proof:lem:avg}

We use the notation in Appendix~\ref{sec:proof:notation}.
Here we leave out the subscript $t$.
By Assumption~\ref{assumption:avg},
we have $(1- \eta ) \A^T \A \preceq \A^T \S_i \S_i^T \A \preceq (1+ \eta ) \A^T \A$.
It follows that
\begin{small}
	\begin{align*}
	(1- \eta ) \H
	\; \preceq \; \tilde{\H}_i
	\; \preceq \; (1+\eta ) \H ,
	\end{align*}
\end{small}%
Thus there exists a matrix $\Ups_i$ satisfying
\[
\H^{\frac{1}{2}} \tilde{\H}_i^{-1} \H^{\frac{1}{2}} \; \triangleq \;
\I_d + \Ups_i
\qquad \textrm{ and } \qquad
-\tfrac{\eta}{1+ \eta} \I_d \preceq \Ups_i \preceq \tfrac{\eta}{1-\eta} \I_d .
\]
By the definitions of $\tilde{\pp}_i$ and $\pp^\star$, we have that
\begin{small}
\begin{align*}
& \H^{\frac{1}{2}} \big( \tilde{\pp}_i - \pp^\star \big) 
\; = \; \H^{\frac{1}{2}} 
\big( \tilde{\H}_i^{-1} - \H^{-1}  \big) \g 
\; = \; \H^{\frac{1}{2}} \H^{-1}
\big( \H - \tilde{\H}_i \big) 
\tilde{\H}_i^{-1} \g \\
& = \; \underbrace{ \big[ \H^{-\frac{1}{2}} \big( \H - \tilde{\H}_i \big) \H^{-\frac{1}{2}} \big]}_{\triangleq \mathbf{\Gamma}_i} 
\underbrace{ \big( \H^{\frac{1}{2}} \tilde{\H}_i^{-1}  \H^{\frac{1}{2}} \big)}_{\triangleq \I_d + \mathbf{\Upsilon}_i} \big]
\big( \H^{-\frac{1}{2}} \g \big) \\
& = \;  \Gam_i
\big( \I_d + \Ups_i \big)
\big( \H^{\frac{1}{2}} \pp^\star \big) ,
\end{align*}
\end{small}%
where the second equality follows from that $\R^{-1} - \T^{-1} = \T^{-1} (\T - \R ) \R^{-1}$
for nonsingular matrices $\R$ and $\T$.
It follows that
\begin{small}
	\begin{align} \label{proof:lem:avg:1}
	& \Big\| \H^{\frac{1}{2}} \big( \tilde{\pp} - \pp^\star \big)  \Big\|_2 
	\; \leq \; \Big\| \frac{1}{m} \sum_{i=1}^m \Gam_i \big( \I_d + \Ups_i \big) \Big\|_2
	\Big\| \H^{\frac{1}{2}} \pp^\star \Big\|_2 \nonumber \\
	& \leq \bigg(  \Big\| \frac{1}{m} \sum_{i=1}^m \Gam_i  \Big\|_2 
	+ \frac{1}{m} \sum_{i=1}^m  \big\| \Gam_i \big\|_2 \big\| \Ups_i \big\|_2 \bigg) \,
	\Big\| \H^{\frac{1}{2}} \pp^\star \Big\|_2 ,
	\end{align}
\end{small}%
It follows from Assumption~\ref{assumption:avg} that
\begin{small}
	\begin{eqnarray*}
		& \big\| \Gam_i \big\|_2 \; 
		\leq \; 
		\eta \, \Big\| \big( \A^T \A + \M \big)^{- \frac{1}{2}}
		\big( \A^T \A \big)  \big( \A^T \A + \M \big)^{- \frac{1}{2}} \Big\|_2  , \\
		& \Big\| \frac{1}{m} \sum_{i=1}^m \Gam_i \Big\|_2 
		\; \leq \; 
		\frac{\eta}{\sqrt{m}} \, \Big\| \big( \A^T \A + \M \big)^{- \frac{1}{2}}
		\big( \A^T \A \big)  \big( \A^T \A + \M \big)^{- \frac{1}{2}} \Big\|_2  .
	\end{eqnarray*}
\end{small}%
Let $\A = \U \Si \V$ be the thin SVD ($\Si$ is $d\times d$). It holds that
\begin{small}
	\begin{align*}
	& \big( \A^T \A + \M \big)^{- \frac{1}{2}}
	\big( \A^T \A \big)  \big( \A^T \A + \M \big)^{- \frac{1}{2}} \\
	& = \;
	\V ( \Si^2 + \V^T \M \V )^{-1/2} \Si^2 ( \Si^2 + \V^T \M \V )^{-1/2} \V^T \\
	& = \V \Si^{-1} \big[ \Si ( \Si^2 + \V^T \M \V )^{-1/2} \Si \big]^2 \Si^{-1} \V^T \\
	& \preceq \; \V \Si^{-1} \big[ \Si ( \Si^2 + \sigma_{\min} (\M) \I_{\rho} )^{-1/2} \Si \big]^2 \Si^{-1} \V^T \\
	& = \; \V \Si^2 \big[ \Si^2 + \sigma_{\min} (\M) \I_\rho  \big]^{-1} \V^T
	\; \preceq \; \tfrac{ \sigma_{\max} (\A_t^T \A_t ) }{ \sigma_{\max} (\A_t^T \A_t ) + \sigma_{\min} (\M) } \I_d
	\; \triangleq \; \vartheta \I_d .
	\end{align*}
\end{small}%
It follows that
\begin{small}
	\begin{eqnarray} \label{proof:lem:avg:4}
	\big\| \Gam_i \big\|_2 \; 
	\leq \; \vartheta \eta  
	\qquad
	\textrm{and}
	\qquad 
	\Big\| \frac{1}{m} \sum_{i=1}^m \Gam_i \Big\|_2 
	\; \leq \; \frac{\vartheta \eta}{\sqrt{m}}  .
	\end{eqnarray}
\end{small}%
It follows from \eqref{proof:lem:avg:1} and \eqref{proof:lem:avg:4} that
\begin{small}
	\begin{align} \label{proof:lem:avg:5} 
	& \big\| \H^{\frac{1}{2}} \big( \tilde{\pp} - \pp^\star \big)  \big\|_2 
	\; \leq \; \vartheta \big( \tfrac{ \eta }{ \sqrt{m} } + \tfrac{ \eta^2 }{ 1 - \eta } \big) \, 
	\big\| \H^{\frac{1}{2}} \pp^\star \big\|_2 ,
	\end{align}
\end{small}%

By the definition of $\phi (\pp)$ and $\pp^\star$, it can be shown that
\[
\phi (\pp^\star )
\; = \; - \big\| \H^{\frac{1}{2}} \pp^\star \big\|_2^2 .
\]
For any $\pp \in \RB^d$, it holds that
\begin{align} \label{proof:lem:avg:6}
& \phi ({\pp} ) - \phi (\pp^\star )
\; = \; \big\| \H^{\frac{1}{2}} {\pp} \big\|_2^2 - 2 \g^T {\pp}
+ \big\| \H^{-\frac{1}{2}} \g \big\|_2^2 \nonumber \\
& = \; \big\| \H^{\frac{1}{2}} {\pp} - \H^{-\frac{1}{2}} \g \big\|_2^2 
\; = \; \big\| \H^{\frac{1}{2}} ({\pp} - \pp^\star ) \big\|_2^2 .
\end{align}
It follows from \eqref{proof:lem:avg:5} and \eqref{proof:lem:avg:6} that
\begin{align*}
& \phi (\tilde{\pp} ) - \phi (\pp^\star )
\; = \; \Big\| \H^{\frac{1}{2}} (\tilde{\pp} - \pp^\star ) \Big\|_2^2 \\
& \leq \;\vartheta^2 \big( \tfrac{\eta}{\sqrt{m}}  + \tfrac{\eta^2 }{1- \eta } \big)^2
\Big\| \H^{\frac{1}{2}} \pp^\star \Big\|_2^2
\; = \; -\vartheta^2 \big( \tfrac{\eta}{\sqrt{m}}  + \tfrac{\eta^2 }{1- \eta } \big)^2 \, \phi (\pp^\star ) ,
\end{align*}
by which the lemma follows.

\section{Proof of Lemma~\ref{lem:avg2} (Effect of Inexact Solution)} \label{proof:lem:avg2}

We use the notation in Appendix~\ref{sec:proof:notation}.
We leave out the subscript $t$.
Let us first bound $\big\| \H^{\frac{1}{2}} (\tilde{\pp}_i' - \tilde{\pp}_i) \big\|_2^2$.
By Assumption~\ref{assumption:avg} and that $\tilde{\pp}_i'$ is close to $ \tilde{\pp}_i$,
we obtain
\begin{align*}
& \big\| \H^{\frac{1}{2}} (\tilde{\pp}_i' - \tilde{\pp}_i) \big\|_2^2
\; = \; (\tilde{\pp}_i' - \tilde{\pp}_i)^T  \H (\tilde{\pp}_i' - \tilde{\pp}_i)\\
& \leq \; \tfrac{1}{1-\eta }
(\tilde{\pp}_i' - \tilde{\pp}_i)^T  
\tilde{\H}_i (\tilde{\pp}_i' - \tilde{\pp}_i) 
\; \leq \; \tfrac{\epsilon_0^2 }{1-\eta }
\tilde{\pp}_i^T  
\tilde{\H}_i \tilde{\pp}_i .
\end{align*}
By definition, $\tilde{\pp}_i = \tilde{\H}_i^{-1}  \g$, it follows that
\begin{align*}
& \big\| \H^{\frac{1}{2}} (\tilde{\pp}_i' - \tilde{\pp}_i) \big\|_2^2 
\; \leq \;  \tfrac{\epsilon_0^2 }{1-\eta }
\g^T
\tilde{\H}_i^{-1}  
\tilde{\H}_i
\tilde{\H}_i^{-1} \g \\
& = \;  \tfrac{\epsilon_0^2 }{1-\eta }
\g^T
\tilde{\H}_i^{-1}  
\g 
\; \leq \; \tfrac{\epsilon_0^2  }{(1 - \eta )^2  }
\g^T \H^{-1} \g 
\; = \; \tfrac{\epsilon_0^2  }{(1 - \eta )^2 }
\big\| \H^{1/2} \pp^\star \big\|_2^2 .
\end{align*}
It follows from the triangle inequality that
\begin{small}
	\begin{align*}
	& \Big\| \H^{\frac{1}{2}} \big( \tilde{\pp}' - \pp^\star \big) \Big\|_2
	\; \leq \; \Big\|  \H^{\frac{1}{2}} \big( \tilde{\pp} - \pp^\star \big) \Big\|_2
	+ \Big\|  \H^{\frac{1}{2}} \big( \tilde{\pp}' - \tilde{\pp}  \big) \Big\|_2 \\
	& = \; \Big\|  \H^{\frac{1}{2}} \big( \tilde{\pp} - \pp^\star \big) \Big\|_2
	+ \Big\|\frac{1}{m} \sum_{i=1}^m  \H^{\frac{1}{2}}
	\big( \tilde{\pp}_i' - \tilde{\pp}_i \big) \Big\|_2 \\
	& \leq \; \Big\| \H^{\frac{1}{2}} \big( \tilde{\pp} - \pp^\star \big) \Big\|_2
	+ \frac{1}{m} \sum_{i=1}^m
	\Big\| \H^{\frac{1}{2}} \big( \tilde{\pp}_i' - \tilde{\pp}_i  \big) \Big\|_2 \\
	& \leq \; \Big\|  \H^{\frac{1}{2}} \big( \tilde{\pp} - \pp^\star \big) \Big\|_2
	+  \tfrac{ \epsilon_0  }{1-\eta  } 
	\Big\|  \H^{\frac{1}{2}} \pp^\star \Big\|_2 \\
	& \leq \; \Big( \vartheta \tfrac{\eta}{\sqrt{m}}  + \vartheta \tfrac{\eta^2 }{1- \eta } 
	+ \tfrac{ \epsilon_0  }{1-\eta  }  \Big)
	\Big\|  \H^{\frac{1}{2}} \pp^\star \Big\|_2 ,
	\end{align*}
\end{small}%
where the last inequality follows from \eqref{proof:lem:avg:5}.
Finally, by \eqref{proof:lem:avg:6}, we obtain
\begin{eqnarray*}
	\phi (\tilde{\pp}' ) - \phi (\pp^\star )
	& = & \big\| \H^{\frac{1}{2}} (\tilde{\pp}' - \pp^\star ) \big\|_2^2 \\
	& \leq & \Big( \vartheta\tfrac{\eta}{\sqrt{m}}  + \vartheta\tfrac{\eta^2 }{1- \eta } 
	+ \tfrac{ \epsilon_0  }{1-\eta  }  \Big)^2
	\Big\| \H^{\frac{1}{2}} \pp^\star \Big\|_2^2 \\
	& = & - \Big( \vartheta\tfrac{\eta}{\sqrt{m}}  + \vartheta\tfrac{\eta^2 }{1- \eta } 
	+ \tfrac{ \epsilon_0  }{1-\eta  }  \Big)^2
	\, \phi (\pp^\star )  ,
\end{eqnarray*}
by which the lemma follows.

\section{Proof of Lemma~\ref{lem:convergence} (Convergence of GIANT)} \label{sec:proof:convergence}

Recall the definitions: $\phi_t (\pp) \triangleq \pp^T \H_t \pp - 2 \g_t^T \pp$,
$\w_{t+1} = \w_{t} - \tilde{\pp}_t$, 
$\De_t = \w_t - \w^\star$, and $\De_{t+1} = \w_{t+1}  - \w^\star$.
It holds that
\[
\tilde{\pp}_t 
\; = \; \w_t - \w_{t+1}
\; = \; \De_{t} - \De_{t+1} .
\]
It follows that
\begin{align*}
& \phi_t \big(\tilde\pp_{t}  \big)
\; = \; (\De_t - \De_{t+1})^T \H_t (\De_t - \De_{t+1}) - 2(\De_t - \De_{t+1})^T \g_t , \\
& (1-\alpha^2 ) \,\cdot \, \phi_t \big(\tfrac{1}{1-\alpha^2} \De_{t}  \big)
\; = \; \tfrac{1}{1 - \alpha^2 } \De_t^T \H_t \De_t - 2  \De_t^T \g_t .
\end{align*}
By taking the difference bewteen the above two equations,
we obtain
\begin{align*}
& \phi_t \big(\tilde\pp_{t}  \big) - 
(1-\alpha^2 )  \, \phi_t \big(\tfrac{1}{1-\alpha^2} \De_{t}  \big) \\
& = \; \De_{t+1}^T \H_t \De_{t+1} - 2 \De_t^T \H_t \De_{t+1} + 2 \De_{t+1}^T \g_t
- \tfrac{\alpha^2 }{1 - \alpha^2 } \De_t^T \H_t \De_t .
\end{align*}
By assumption, it holds that
\begin{eqnarray*}
\phi_t \big(\tilde\pp_{t}  \big)
\; \leq \; (1- \alpha^2 ) \, \min_{\pp} \phi_t \big(\pp  \big)
\; \leq \; (1-\alpha^2 )  \, \phi_t \big(\tfrac{1}{1-\alpha^2} \De_{t}  \big) ,
\end{eqnarray*}
and thereby
\begin{eqnarray}\label{eq:lem:convergence:1}
\De_{t+1}^T \H_t \De_{t+1} - 2 \De_t^T \H_t \De_{t+1} + 2 \De_{t+1}^T \g_t
- \tfrac{\alpha^2 }{1 - \alpha^2 } \De_t^T \H_t \De_t
\; \leq \; 0.
\end{eqnarray}
We can write $\g_t \triangleq \g (\w_t)$ by
\begin{eqnarray*}
	\g (\w_t)
	& = & \g (\w^\star ) + \Big( \int_0^1 \nabla^2 f \big( \w^\star + \tau (\w_t - \w^\star ) \big) d \tau \Big) (\w_t - \w^\star) \nonumber \\
	& = & \Big( \int_0^1  \nabla^2 f \big( \w^\star + \tau (\w_t - \w^\star ) \big) d \tau \Big) \De_t ,
\end{eqnarray*}
where the latter equality follows from that $\g (\w^\star ) = \0$.
It follows that
\begin{eqnarray}  \label{eq:lem:convergence:2}
\Big\|  \H_t \De_{t} - \g (\w_{t}) \Big\|_2
& \leq & \big\| \De_{t} \big\|_2 \,
\bigg\| \int_0^1 \Big[  \nabla^2 f (\w_{t}) -  \nabla^2 f \big( \w^\star + \tau (\w_t - \w^\star ) \big)  \Big] d \tau \bigg\|_2 \nonumber \\
& \leq & \big\| \De_{t} \big\|_2 \,
\int_0^1 \Big\|  \nabla^2 f (\w_{t}) -  \nabla^2 f \big( \w^\star + \tau (\w_t - \w^\star ) \big)  \Big\|_2 d \tau \nonumber \\
& \leq & \big\| \De_{t} \big\|_2 \,
\int_0^1 (1-\tau ) L  \big\| \w_t - \w^\star \big\|_2 d \tau \nonumber \\
& = & \tfrac{L}{2} \big\| \De_{t} \big\|_2^2  .
\end{eqnarray}
Here the second inequality follows from Jensen's inequality;
the third inequality follows from the assumption of $L$-Lipschitz.
It follows from \eqref{eq:lem:convergence:1} and \eqref{eq:lem:convergence:2} that
\begin{eqnarray*}
\De_{t+1}^T \H_t \De_{t+1} 
& \leq &
2 \De_{t+1}^T \big( \H_t \De_{t}  - \g_t \big)
+ \tfrac{\alpha^2 }{1 - \alpha^2 } \De_t^T \H_t \De_t \\
& \leq & L \big\| \De_{t+1}^T \big\|_2  \big\| \De_{t}^T \big\|_2^2
+ \tfrac{\alpha^2 }{1 - \alpha^2 } \De_t^T \H_t \De_t ,
\end{eqnarray*}
by which the lemma follows.